\documentclass[letterpaper]{article} 
\usepackage{aaai25}  
\usepackage{times}  
\usepackage{helvet}  
\usepackage{courier}  
\usepackage[hyphens]{url}  
\usepackage{graphicx} 
\usepackage{natbib}  
\usepackage{caption} 
\usepackage{algorithm}
\usepackage{algorithmic}

\usepackage{newfloat}
\usepackage{listings}
\DeclareCaptionStyle{ruled}{labelfont=normalfont,labelsep=colon,strut=off} 
\floatstyle{ruled}
\newfloat{listing}{tb}{lst}{}
\floatname{listing}{Listing}
\title{Learning Set Functions with Implicit Differentiation}

\author{\added{G\"{o}zde \"{O}zcan,  
        Chengzhi Shi,  
        Stratis Ioannidis\\}
}

\affiliations{
\added{Department of Electrical and Computer Engineering\\
Northeastern University\\
Boston, MA 02115, USA\\
\{gozcan, cshi, ioannidis\}@ece.neu.edu}
}

\usepackage{microtype}      
\usepackage{subfigure}
\usepackage{booktabs} 
\usepackage[utf8]{inputenc} 
\usepackage{amsfonts}       
\usepackage{nicefrac}       
\usepackage{amssymb, amsthm}
\usepackage{amsmath}
\usepackage{multirow}
\usepackage{makecell}

\theoremstyle{plain}
\newtheorem{theorem}{Theorem}[section]

\newtheorem{lemma}[theorem]{Lemma}
\newtheorem{corollary}[theorem]{Corollary}
\theoremstyle{definition}
\newtheorem{definition}[theorem]{Definition}
\newtheorem{assumption}[theorem]{Assumption}
\theoremstyle{remark}

\DeclareMathOperator*{\argmax}{arg\,max}
\DeclareMathOperator*{\argmin}{arg\,min}
\DeclareMathOperator{\diag}{diag}

\usepackage[final, deletedmarkup=em]{changes}

\nocopyright
\newcommand{\fullversion}[2]{#2}

\begin{document}

\maketitle

\begin{abstract}
  \citet{ou2022learning} introduce the problem of learning set functions from data generated by a so-called optimal subset oracle. 
  Their approach 
  approximates the underlying utility function with an energy-based model, whose parameters are estimated via mean-field variational inference. 
  \citet{ou2022learning} show this reduces to fixed point iterations; however, as the number of iterations increases, automatic differentiation quickly becomes computationally prohibitive due to the size of the Jacobians that are stacked during backpropagation. We address this challenge with implicit differentiation and examine the convergence conditions for the fixed-point iterations. We empirically demonstrate the efficiency of our method on synthetic and real-world subset selection applications including product recommendation, set anomaly detection and compound selection tasks.

\end{abstract}
\section{Introduction}
  
Many interesting applications operate with set-valued outputs and/or inputs. Examples include product recommendation \citep{bonab2021cross, schafer1999recommender}, compound selection \citep{ning2011improved}, set matching \citep{saito2020exchangeable}, set retrieval \citep{feng2016pairwise}, point cloud processing \citep{Zhao_2019_CVPR, gionis2001efficient}, set prediction~\citep{zhang2019deep}, and set anomaly detection \citep{mavskova2024deep}, to name a few. Several recent works \citep{zaheer2017deep, lee2019set} apply neural networks to learn set functions from input/function value pairs, assuming access to a dataset generated by a \emph{function value} oracle. In other words, they assume having access to a dataset generated by an oracle that evaluates the value of the set function for any given input set.  

Recently, \citet{ou2022learning} proposed an approximate maximum likelihood estimation framework under the supervision of a so-called \emph{optimal subset} oracle. In contrast to traditional function value oracles, a label produced by an optimal subset oracle is the subset that maximizes an (implicit) utility set function, in the face of several alternatives. The goal of inference is to learn, in a parametric form, this utility function, under which observed oracle selections are optimal.  
As MLE is intractable in this setting, 
\citet{ou2022learning} propose performing variational inference instead. In turn, they show that approximating the distribution of oracle selections requires solving a fixed-point equation per sample. However, these fixed-point iterations may diverge in practice. In addition, 
\citet{ou2022learning} implement these iterations via \emph{loop unrolling}, i.e., by stacking up  neural network layers across iterations, and calculating the gradient using automatic differentiation; 
this makes backpropagation expensive, limiting their experiments to only a handful of iterations.

In this work, we establish a condition under which the fixed-point iterations proposed by 
\citet{ou2022learning} are guaranteed to converge. We also propose a more effective gradient computation utilizing the recent advances in implicit differentiation \citep{bai2019deep, kolter2020implicit, huang2021textrm}, instead of unrolling the fixed-point iterations via automatic differentiation \citep{paszke2017automatic}. This corresponds to differentiating after infinite fixed point iterations, while remaining tractable; we experimentally show that this 
improves the predictive performance of the inferred models.

We make the following contributions:
\begin{itemize}
    \item We prove that, as long as the multilinear relaxation~\cite{calinescu2011maximizing} of the objective function is bounded, and this bound is inversely proportional to the size of the ground set, the fixed-point iterations arising during the MLE framework introduced by \citet{ou2022learning} converge to a unique solution, regardless of the starting point. 
    
    \item We propose a more effective gradient computation by using implicit differentiation instead of unrolling the fixed-point iterations via automatic differentiation. To the best of our knowledge, we are the first ones to propose utilizing implicit differentiation in the context of learning set functions.
    
    \item We conduct experiments to show the advantage of our approach on multiple subset selection applications including set anomaly detection, product recommendation, and compound selection tasks. We also show, in practice, that the fixed-point iterations converge by normalizing the gradient of the multilinear relaxation. 
\end{itemize}

The remainder of the paper is organized as follows. We present related literature in Sec.~\ref{sec:related}. We summarize the learning set functions with optimal subset oracle setting introduced by \citet{ou2022learning} in Sec.~\ref{sec: setup}. We state our main contributions in Sec.~\ref{sec:contributions}. We present our experimental results in Sec.~\ref{sec:exp} and we conclude in Sec.~\ref{sec:conclusions}.

\section{Related Work}\label{sec:related}

\paragraph{Learning Set Functions from Oracles.}
There is a line of work where a learning algorithm is assumed to have access to the value of an unknown utility function for a given set~\cite{feldman2014learning, balcan2018submodular, zaheer2017deep, lee2019set, wendler2021learning, de2022neural}. This is the \emph{function value oracle} setting. \citet{zaheer2017deep}~and~\citet{de2022neural} regress over input set - function value pairs by minimizing the squared loss of the predictions while \citet{lee2019set} minimize the mean absolute error. However, obtaining a function value to a given subset is not an easy task for real-world applications. The value of a set may not be straightforward to quantify or can be expensive to compute. Alternatively,~\citet{tschiatschek2018differentiable}~and~\citet{ou2022learning} assume having access to an \emph{optimal subset oracle} for a given ground set in the training data. Similarly, we do not learn the objective function explicitly from input set - output value pairs. We learn it implicitly in the optimal subset oracle setting.

\paragraph{Learning Set Functions with Neural Networks.}
Multiple works aim to extend the capabilities of neural networks for functions on discrete domains, i.e., \emph{set functions}~\cite{zaheer2017deep, wendler2019powerset, soelch2019on, lee2019set, wagstaff2019limitations, Kim_2021_CVPR, zhang2022set, giannone2022scha}. Diverging from the traditional paradigm where the input data is assumed to be in a fixed dimensional vector format, set functions are characterized by their \emph{permutation invariance}, i.e., the output of a set does not depend on the order of its elements. We refer the reader to a survey about permutation-invariant networks by \citet{kimura2024permutation} for a more detailed overview. 
In this work, we also enforce permutation invariance by combining the energy-based model in Sec.~\ref{sec:ebm} with deep sets~\cite{zaheer2017deep}, following the proposed method of \citet{ou2022learning} (see also \fullversion{App.~A \added{of \citet{ozcan2024learning}}}{App.~\ref{app:perm_inv}}).




\citet{karalias2022neural} integrate neural networks with set functions by leveraging extensions of these functions to the continuous domain. Note that, their goal is not to learn a set function but to learn \emph{with} a set function, which differs from our objective.


\paragraph{Learning Submodular Functions.}
It is common to impose some structure on the objective when learning set functions. The underlying objective is often assumed to be submodular, i.e., it exhibits a diminishing returns property, while the parameters of such function are typically learned from function value oracles~\cite{dolhansky2016deep, bilmes2017deep, djolonga2017differentiable, kothawade2020deep, de2022neural, bhatt2024deep, gomez2012inferring, bach2013learning, feldman2014learning, he2016learning}. We do not make such assumptions, therefore, our results are applicable to a broader class of set functions. 




\paragraph{Implicit Differentiation.} In the context of machine learning, implicit differentiation 
is used in hyperparameter optimization~\cite{lorraine2020optimizing, bertrand2020implicit}, optimal control~\cite{xu2024revisiting}, reinforcement learning~\cite{nikishin2022control}, bi-level optimization~\cite{arbel2022amortized, zucchet2022beyond}, neural ordinary differential equations~\cite{chen2018neural, li2020scalable} and set prediction~\cite{zhang2022multiset}, to name a few. Inspired by the advantages observed over this wide-range of problems, we use implicit differentiation, i.e., a method for differentiating a function that is given implicitly \citep{krantz2002implicit}, to learn set functions for subset selection tasks by leveraging the JAX-based, modular automatic implicit differentiation tool provided by~\citet{blondel2022efficient}.

\paragraph{Implicit Layers.} Instead of specifying the output of a deep neural network layer as an explicit function over its inputs, \emph{implicit layers} are specified implicitly, via the conditions that layer outputs and inputs must jointly satisfy~\citep{kolter2020implicit}. Deep Equilibrium Models (DEQs)~\cite{bai2019deep} and their variants~\cite{winston2020monotone, huang2021textrm, sittoni2024subhomogeneous} directly compute the fixed-point resulting from stacking up hidden implicit layers by black-box root-finding methods, while also directly differentiating through the stacked fixed-point equations via implicit differentiation. We adapt this approach when satisfying the fixed-point constraints arising in our setting. The main difference is that in the aforementioned works, implicit layers correspond to a weight-tied feedforward network while in our case, they correspond to a deep set~\cite{zaheer2017deep} style architecture.

\section{Problem Setup}\label{sec: setup}

In the setting introduced by 
\citet{ou2022learning}, the aim is to learn set functions from a dataset generated by a so-called \emph{optimal subset} oracle. The dataset $\mathcal{D}$ consists of sample pairs of the form $(S^*,V)$, where (query) $V\subseteq \Omega$ is a set of \emph{options}, i.e., items from a universe $\Omega$ and (response) $S^*$ is the \emph{optimal subset} of $V$, as selected by an oracle. We further assume that each item is associated with a feature vector of dimension $d_f$, i.e., $\Omega\subseteq \mathbb{R}^{d_f}$.   
The goal is to learn a set function $F_{\boldsymbol{\theta}}: 2^\Omega\times 2^\Omega \rightarrow \mathbb{R}$, parameterized by $\boldsymbol{\theta} \in \mathbb{R}^d$, modeling the utility of the oracle, so that 
\begin{equation}\label{eq:optset}
    S^* = \mathop{\argmax}_{S \subseteq V} F_{\boldsymbol{\theta}} (S, V),
\end{equation}
for all pairs $(S^*,V)\in \mathcal{D}$.
As a motivating example, consider the case of product recommendations. Given a ground set $V$ of possible products to recommend, a recommender selects an optimal subset $S^* \subseteq V$ 
and suggests these to a user. In this setting, the 
function  $F_{\boldsymbol{\theta}} (S, V)$ captures, e.g., the recommender objective, the utility of the user, etc. Having access to a dataset of such pairs, the goal is to learn $F_{\boldsymbol{\theta}}$, effectively reverse-engineering the objective of the recommender engine, inferring the user's preferences, etc.


\subsection{MLE with Energy-Based Modeling}\label{sec:ebm}
\citet{ou2022learning} propose an approximate maximum likelihood estimation (MLE) by modeling oracle behavior via a Boltzmann energy (i.e., soft-max) model \citep{murphy2012machine, mnih2005learning, hinton2006unsupervised, lecun2006tutorial}. 
They assume that the oracle selection is probabilistic, and the probability that $S$ is selected given options $V$ is given by:
\begin{equation}\label{eq:EBM}
    p_{\boldsymbol{\theta}} (S\,|\,V) = 
    \frac{\exp{(F_{\boldsymbol{\theta}}(S,\, V))}}{\sum_{S' \subseteq V} \exp{(F_{\boldsymbol{\theta}}(S',\, V))}}.
\end{equation}
This is equivalent to Eq.~\eqref{eq:optset}, presuming that the utility $F_{\boldsymbol{\theta(\cdot)}}$ is distorted by Gumbel noise~\citep{kirsch2023stochastic}.
Then, given a dataset $\mathcal{D} = \{( S_i^*, V_i)\}_{i=1}^N$, MLE amounts to:
\begin{equation}
    \begin{split}
        &\argmax_{\boldsymbol{\theta}} 
        \sum_{i=1}^N \left[\log p_{\boldsymbol{\theta}} (S_i^*\,|\,V_i)\right].
    \end{split}
\end{equation}
Notice that multiplying $F_{\boldsymbol{\theta}}$ with a constant $c > 0$ makes no difference in the behavior of the optimal subset oracle in Eq.~\eqref{eq:optset}: the oracle would return the same decision under arbitrary re-scaling. However,  using  $c\cdot F_{\boldsymbol{\theta}}(\cdot)$ in the energy-based model of Eq.~\eqref{eq:EBM} corresponds to 
 setting a temperature parameter $c$ in the Boltzmann distribution~\citep{murphy2012machine, kirsch2023stochastic},  interpolating between the deterministic selection ($c\to\infty$) in Eq.~\eqref{eq:optset} and the uniform distribution ($c\to 0$).\footnote{From a Bayesian point of view, multiplying $F_{\boldsymbol{\theta}}(\cdot)$ with $c>0$ yields the posterior predictive distribution under an uninformative Dirichlet conjugate prior per set with parameter $\alpha=e^c$ \citep{murphy2012machine}.}

\subsection{Variational Approximation of Energy-Based Models}
Learning $\boldsymbol{\theta}$ by MLE 
is 
challenging precisely due to the exponential number of terms in the denominator of Eq.~\eqref{eq:EBM}. Instead, 
\citet{ou2022learning} construct an alternative optimization objective via mean-field variational inference as follows. First, they introduce a mean field variational approximation of the density $p_{\boldsymbol{\theta}}$ given by $q(S, V, \boldsymbol{\psi}) = \prod_{j \in S} \psi_j \prod_{j \in V \setminus S} (1 - \psi_j)$,  parameterized by the probability vector $\boldsymbol{\psi}$: this represents the probability that each element $j \in V$ is in the optimal subset $S^*$. Then,  estimation via variational inference amounts to the following optimization problem:
\begin{align}\label{eq:loss}
        \mathop{\text{Min.}}
        \quad &\mathcal{L}(\{\boldsymbol{\psi}^*_i\}
        ) = \mathbb{E}_{\mathbb{P}(V,\, S)} [-\log q(S, \,V\, ,\boldsymbol{\psi}^*)] \approx \\
        & \frac{1}{N} \sum_{i=1}^N \left(-\!\sum_{j \in S_i^*} \log \psi_{ij}^* -\!\!\!\! \sum _{j \in V_i \setminus S_i^*} \!\!\!\log \left(1 - \psi_{ij}^*\right)\right),\nonumber\\
        \text{subj.~to}\quad~&\boldsymbol{\psi}^*_i = \argmin_{\boldsymbol{\psi}} \mathbb{KL} (q(S_i,\, V_i,\,\boldsymbol{\psi}) \,||\, p_{\boldsymbol{\theta}} (S_i\mid V_i)) ,\nonumber\\&\text{for all}~i\in\{1,\ldots,n\},\nonumber
\end{align}
where 
$\boldsymbol{\psi}_i^* \in [0, 1]^{|V|}$ is the probability vector of elements in $V_i$ being included in $S_i$,  
$\mathbb{KL}(\cdot||\cdot)$ is the Kullback-Leibler divergence, and $p_{\boldsymbol{\theta}}(\cdot)$ is the energy-based model defined in Eq.~\eqref{eq:EBM}. 
In turn, this is found through the ELBO maximization process we discuss in the next section.





\subsection{ELBO Maximization}\label{sec:elbo_max} 

\begin{algorithm}[!t] 
\caption{$\mathtt{DiffMF}$~\citep{ou2022learning}}\label{alg:diffMF}
    \begin{algorithmic}[1]
        \REQUIRE training dataset $\{(S_i^\ast, \, V_i)\}_{i=1}^N$, learning rate $\eta$, number of samples $m$  
        \ENSURE parameter $\boldsymbol{\theta}$
        \STATE $\boldsymbol{\theta} \gets$  initialize
        \REPEAT 
            \STATE sample training data point
                
                $(S^\ast, V) \sim \{(S_i^\ast,\, V_i)\}_{i=1}^N$
            \STATE initialize the variational parameter
            
                $\boldsymbol{\psi}^{(0)} \leftarrow 0.5 * \mathbf{1}$
    
            \REPEAT
                \FOR {$k \gets 1, \ldots, K$}
                    \FOR {$j \gets 1, \ldots, |V|$ in parallel}
                        \STATE sample $m$ subsets
                        
                        $S_{\ell} \sim q(S, (\boldsymbol{\psi}^{(k-1)} | \psi_j^{(k-1)} \gets 0))$
                        
                        \STATE update variational parameter 
                        
                        $\mathbf{\psi}_j^{(k)}\!\gets\! \sigma \left(\frac{1}{m}\sum_{\ell=1}^m \!\left[F_{\boldsymbol{\theta}} (S_{\ell} \cup\{ j\}) \!-\! F_{\boldsymbol{\theta}}(S_{\ell})\right]\right)$
                        
                    \ENDFOR
                \ENDFOR
            \UNTIL convergence of $\boldsymbol{\psi}$
            
    
            
            
            
            \STATE update parameter $\boldsymbol{\theta}$ by unfolding the derivatives of the $K-$layer meta-network resulting from the fixed-point equations given in Eq.~\eqref{eq:iterative} during SGD\label{line:update}

                $\partial_{\boldsymbol{\theta}} \boldsymbol{\psi}^{(K)} (\boldsymbol{\theta}) \gets$ 
                
                $\quad \partial_{\boldsymbol{\theta}} \underbrace{\boldsymbol{\sigma} (\nabla_{\boldsymbol{\psi}^{(K-1)}} \Tilde{F} (\ldots (\boldsymbol{\sigma} (\nabla_{\boldsymbol{\psi}^{(0)}} \Tilde{F} (\boldsymbol{\psi}^{(0)})) \ldots))}_{K \text{nested functions}} $

                $\nabla_{\boldsymbol{\theta}} \mathcal{L} (\boldsymbol{\psi}^{{(K)}}, \boldsymbol{\theta}) \gets \nabla_{\boldsymbol{\psi}^{(K)}} \mathcal{L}(\boldsymbol{\psi}^{(K)}(\boldsymbol{\theta})) \cdot \partial_{\boldsymbol{\theta}} \boldsymbol{\psi}^{(K)} (\boldsymbol{\theta})$
            
                $\boldsymbol{\theta} \gets \boldsymbol{\theta} - \eta \nabla_{\boldsymbol{\theta}} \mathcal{L} \left(\boldsymbol{\psi}^{(K)}, \boldsymbol{\theta}\right)$
        \UNTIL convergence of $\boldsymbol{\theta}$
    \end{algorithmic}
\end{algorithm}

To compute $\boldsymbol{\psi}^*$, 
~\citet{ou2022learning} show that minimizing the constraint 
in Eq.~\eqref{eq:loss} via maximizing the corresponding evidence lower bound (ELBO) reduces to solving a fixed point equation. In particular, omitting the dependence on $i$ for brevity, the constraint in Eq.~\eqref{eq:loss} is equivalent to the following  ELBO  maximization \citep{kingma2013auto, blei2017variational}: 
\begin{equation}\label{eq:ELBO}
    \max_{\boldsymbol{\psi}} \Tilde{F} (\boldsymbol{\psi}, \boldsymbol{\theta}) + \mathbb{H} (q(S, V, \boldsymbol{\psi}))
    ,
\end{equation}
where $\mathbb{H}(\cdot)$ 
is the entropy and  $\Tilde{F}: [0, 1]^{|V|} \times \mathbb{R}^d \rightarrow \mathbb{R}$ is the so-called multilinear extension of $F_{\boldsymbol{\theta}}(S,V)$ \cite{calinescu2011maximizing}, given by:
\begin{equation}\label{eq:mult_ext}
    \begin{split}
        \Tilde{F} (\boldsymbol{\psi}, \boldsymbol{\theta}) 
        &= \sum_{S \subseteq V} F_{\boldsymbol{\theta}}(S,V) \prod_{j \in S} \psi_j \prod_{j \in V\setminus S} (1-\psi_j).
    \end{split} 
\end{equation}
\citet{ou2022learning} show that a stationary point maximizing the ELBO in Eq.~\eqref{eq:ELBO} must satisfy:
\begin{equation} \label{eq:fixed_point}
    \begin{split}
        \boldsymbol{\psi} - \boldsymbol{\sigma}(\nabla_{\boldsymbol{\psi}} \Tilde{F} (\boldsymbol{\psi}, \boldsymbol{\theta})) &= 0,
    \end{split}
\end{equation}
where the function $\boldsymbol{\sigma}: \mathbb{R}^{|V|} \rightarrow \mathbb{R}^{|V|}$ is defined as $\boldsymbol{\sigma} (\mathbf{x}) = \begin{bmatrix}\sigma (x_j)\end{bmatrix}_{j=1}^{|V|}$ and $\sigma: \mathbb{R} \rightarrow \mathbb{R}$ is the sigmoid function, i.e., $\sigma(x) = (1 + \exp{(-x)})^{-1}$. The detailed derivation of this condition can be found in App.~\fullversion{C.1
\added{ of \citet{ozcan2024learning}}}{\ref{app:fixed_point_derivation}}. 

Observing that the stationary condition in Eq.~\eqref{eq:fixed_point} is a fixed point equation, 
\citet{ou2022learning} propose solving it via 
the following fixed-point iterations. Given $\boldsymbol{\theta}\in\mathbb{R}^d$,
\begin{align}\label{eq:iterative}
    \begin{split}
       \boldsymbol{\psi}^{(0)} &\leftarrow \text{Initialize in } [0, 1]^{|V|},\\
        \boldsymbol{\psi}^{(k)} &\leftarrow \boldsymbol{\sigma} (\nabla_{\boldsymbol{\psi}} \Tilde{F} (\boldsymbol{\psi}^{(k-1)}, \boldsymbol{\theta})),\\
       \boldsymbol{\psi}^* &\leftarrow \boldsymbol{\psi}^{(K)},
    \end{split}
\end{align}
where $k\in \mathbb{N},$ and $K$ is the number of iterations. 
The exact computation of the multilinear relaxation defined in Eq.~\eqref{eq:mult_ext} requires an exponential number of terms in the size of $V$. However, it is possible to efficiently estimate both the multilinear relaxation and its gradient $\nabla_{\boldsymbol{\psi}} \Tilde{F} (\boldsymbol{\psi}, \boldsymbol{\theta})$ via Monte Carlo 
sampling (see App.~\fullversion{C.2
\added{ of \citet{ozcan2024learning}}}{\ref{app:sampling}} for details). 

\subsection{$\mathtt{DiffMF}$ and Variants}

Putting everything together yields the  $\mathtt{DiffMF}$ algorithm introduced by 
\citet{ou2022learning}. For completeness, we summarize this procedure in Alg.~\ref{alg:diffMF}. In short, they implement the fixed-point iterative update steps in Eq.~\eqref{eq:iterative} by executing a fixed number of iterations $K$, given $\boldsymbol{\theta}$, and unrolling the loop: in their implementation, this amounts to stacking up $K$ layers, each involving an estimate of the gradient of the multilinear relaxation via sampling, and thereby multiple copies of a neural network representing $F_{\boldsymbol{\theta}}(\cdot)$ (one per sample)
. Subsequently, this extended network is entered in the loss given in Eq.~\eqref{eq:loss}, which is minimized w.r.t.~$\boldsymbol{\theta}$ via SGD. 

They also introduce two variants of this algorithm, regressing also $\boldsymbol{\psi}^{(0)}$ as a function of the item features via an extra recognition network, assuming the latter are independent (terming inference in this setting as $\mathtt{EquiVSet}_{\textmd{ind}}$) or correlated by a Gaussian copula \citep{sklar1973random, Nels06} (termed $\mathtt{EquiVSet}_{\textmd{copula}}$). Compared to $\mathtt{DiffMF}$, both translate to additional initial layers and steps per epoch.

\subsection{Challenges}\label{sec:challenges}
The above approach by 
\citet{ou2022learning}, and its variants, have two drawbacks. First, the fixed-point iterative updates given in Eq.~\eqref{eq:iterative} are not guaranteed to converge to an optimal solution. We indeed frequently observed divergence experimentally, in practice. 
Without convergence and uniqueness guarantees, the quality of the output, $\boldsymbol{\psi}^{(K)}$, is heavily dependent on the selection of the starting point, $\boldsymbol{\psi}^{(0)}$.   
Moreover, as these iterations correspond to stacking up layers, each containing multiple copies of $F_{\boldsymbol{\theta}}(\cdot)$ due to sampling, 
backpropagation is computationally prohibitive both in terms of time as well as space complexity. In fact, poor performance due to lack of convergence, as well as computational considerations, led Ou et al. to set the number of iterations to $K\leq 5$  (even $K=1$) in their experiments.
We address both of these challenges in the next section.


\section{Our Approach}\label{sec:contributions}
Recall from the previous section that minimizing the constraint of the optimization problem given in Eq.~\eqref{eq:loss} is the equivalent of the ELBO in Eq.~\eqref{eq:ELBO}, and the stationary condition of optimizing this ELBO reduces to Eq.~\eqref{eq:fixed_point}. Stitching everything together, we wish to solve the following optimization problem:
\begin{align}\label{eq:fp_loss}
        \mathop{\text{Min.}}_{\{\boldsymbol{\psi}_i^\ast\}, \boldsymbol{\theta}} \quad
        & \mathcal{L}(\{\boldsymbol{\psi_i^\ast}\}
        ) \approx \\
        & \frac{1}{N} \sum_{i=1}^N \left(-\sum_{j \in S_i^*} \log \psi_{ij} -\!\!\!\! \sum _{j \in V_i \setminus S_i^*} \log \left(1 - \psi_{ij}\right)\right),\nonumber\\
        \text{subj.~to}\quad
        & \boldsymbol{\psi}_i^{\ast} = \boldsymbol{\sigma}(\nabla_{\boldsymbol{\psi}} \Tilde{F} (\boldsymbol{\psi}_i^\ast, \boldsymbol{\theta})),~\text{for all}~i\in\{i,\ldots,n\}.\nonumber
\end{align}

To achieve this goal, we 
 (a) establish conditions under which iterations of  Eq.~\eqref{eq:iterative} converge to a unique solution, by utilizing the Banach fixed-point theorem and (b) establish a way to efficiently compute the gradient of the loss at the fixed-point by using the implicit function theorem. Our results pave the way to utilize recent tools developed in the context of implicit differentiation \citep{bai2019deep, kolter2020implicit, blondel2022efficient} to the setting of 
 \citet{ou2022learning}.

\subsection{Convergence Condition for the Fixed-Point}\label{sec:fp_convergence} 
Fixed-points can be attracting, repelling, or neutral \citep{davies2018exploring, rechnitzer2003fpsum}. We characterize the condition under which the convergence is guaranteed in the following assumption.

\begin{assumption}\label{asm:bound}
    Consider 
    the multilinear relaxation $\Tilde{F}: [0, 1]^{|V|} \times \mathbb{R}^d \rightarrow \mathbb{R}$ of $F_{\boldsymbol{\theta}}(\cdot)$, as defined in Eq.~\eqref{eq:mult_ext}. For all $\boldsymbol{\theta} \in \mathbb{R}^d$,
    \begin{equation}\label{eq:cond}
        \sup_{\boldsymbol{\psi} \in [\boldsymbol{0}, \boldsymbol{1}]} |\Tilde{F}(\boldsymbol{\psi}, \boldsymbol{\theta})| < \frac{1}{|V|}.
    \end{equation}
\end{assumption}

As discussed in Sec.~\ref{sec: setup}, scaling $F_{\boldsymbol{\theta}} (S, V)$  by a positive scalar amounts to setting the temperature of a Boltzmann distribution. Moreover, neural networks are often Lipschitz-regularized for bounded inputs and weights \citep{szegedy2014intriguing, virmaux2018lipschitz, gouk2021regularisation}. Therefore, for any such Lipschitz neural network, we can satisfy Asm.~\ref{asm:bound} by appropriately setting the temperature parameter of the EBM in Eq.~\eqref{eq:EBM}. Most importantly, satisfying this condition guarantees convergence: 

\begin{theorem}\label{thm:banach}
    Assume a set function $F_{\boldsymbol{\theta}}: 2^V \rightarrow \mathbb{R}$ satisfies Asm.~\ref{asm:bound}. Then, the fixed-point given in Eq.~\eqref{eq:fixed_point} has a unique solution $\boldsymbol{\psi}^* \in [0, 1]^{|V|}$ where $\boldsymbol{\psi}^* = \boldsymbol{\sigma}(\nabla_{\boldsymbol{\psi}} \Tilde{F} (\boldsymbol{\psi}^*, \boldsymbol{\theta})).$ Moreover, starting with an arbitrary point $\boldsymbol{\psi}^{(0)} \in [0, 1]^{|V|}$, $\boldsymbol{\psi}^*$ can be found via the fixed-point iterative sequence described in Eq.~\eqref{eq:iterative} where $\lim_{k \to \infty} \boldsymbol{\psi}^{(k)} = \boldsymbol{\psi}^*$.
\end{theorem}
The proof can be found in App.~\fullversion{E
\added{ of \citet{ozcan2024learning}}}{\ref{app:convergence}} and relies on the Banach fixed-point theorem~\citep{banach1922operations}. Thm.~\ref{thm:banach} implies that as long as $\Tilde{F}(\boldsymbol{\psi}, \boldsymbol{\theta})$ is bounded and this bound is inversely correlated with the size of the ground set, we can find a unique solution to Eq.~\eqref{eq:fixed_point}, no matter where we start the iterations in Eq.~\eqref{eq:iterative}.


\subsection{Efficient Differentiation through Implicit Layers} 
Our second contribution is to disentangle gradient computation from stacking layers together, by using the implicit function theorem \citep{krantz2002implicit}. This allows us to use the recent work on deep equilibrium models (DEQs) \citep{bai2019deep, kolter2020implicit}.

Define $\boldsymbol{\psi}^*(\cdot)$ to be the map $\boldsymbol{\theta}\mapsto \boldsymbol{\psi}^*(\boldsymbol{\theta})$ induced by Eq.~\eqref{eq:fixed_point}; equivalently, given $\boldsymbol{\theta}$, $\boldsymbol{\psi}^*(\boldsymbol{\theta})$ is the (unique by Thm.~\ref{thm:banach}) limit point of iterations given in Eq.~\eqref{eq:iterative}. Observe that, by the chain rule:
\begin{align}\label{eq: chain rule}
    \nabla_{\boldsymbol{\theta}} \mathcal{L} (\boldsymbol{\psi}^*(\boldsymbol{\theta})
    ) = \nabla_{\boldsymbol{\psi}} \mathcal{L}(\boldsymbol{\psi}^*(\boldsymbol{\theta})
    ) \cdot \partial_{\boldsymbol{\theta}} \boldsymbol{\psi}^* (\boldsymbol{\theta}).
\end{align}
The term that is difficult to compute here via back-propagation, that required stacking in 
\citet{ou2022learning}, is the Jacobian $\partial_{\boldsymbol{\theta}} \boldsymbol{\psi}^* (\boldsymbol{\theta})$, as we do not have the map $\boldsymbol{\psi}^*(\cdot)$ in a closed form. 
Nevertheless, we can use the implicit function theorem (see Thm.~\fullversion{D.4 
in \added{\citet{ozcan2024learning}}}{\ref{thm:implicit} in App.~\ref{app:prelim}\deleted{ of the supplement}}) 
to compute this quantity. 

 Indeed, to simplify the notation for clarity, we define a function $G: [0, 1]^{|V|} \times \mathbb{R}^d \rightarrow [0, 1]^{|V|}$, where $$G(\boldsymbol{\psi}(\boldsymbol{\theta}), \boldsymbol{\theta}) \triangleq \boldsymbol{\sigma}(\nabla_{\boldsymbol{\psi}} \Tilde{F} (\boldsymbol{\psi}, \boldsymbol{\theta})) - \boldsymbol{\psi}$$ 
 and rewrite Eq.~\eqref{eq:fixed_point} as $G(\boldsymbol{\psi}(\boldsymbol{\theta}), \boldsymbol{\theta}) = 0$.
Using the implicit function theorem, given in App.~\fullversion{D
\added{ of \citet{ozcan2024learning}}}{\ref{app:prelim} \deleted{of the supplement}}, we obtain
    \begin{equation}\label{eq:fp_implicit}
        \underbrace{-\partial_{\boldsymbol{\psi}} G(\boldsymbol{\psi}^*(\boldsymbol{\theta}), \boldsymbol{\theta)}}_{A \in \mathbb{R}^{|V| \times |V|}} \underbrace{\partial_{\boldsymbol{\theta}} \boldsymbol{\psi}^*(\boldsymbol{\theta})}_{J \in \mathbb{R}^{|V| \times d}} = \underbrace{\partial_{\boldsymbol{\theta}} G(\boldsymbol{\psi}^*(\boldsymbol{\theta}), \boldsymbol{\theta)}}_{B \in \mathbb{R}^{|V| \times d}}.
    \end{equation}
    This yields the following way of computing the Jacobian via implicit differentiation:
\begin{theorem}\label{thm:implicit_result}
    Computing $\partial_{\boldsymbol{\theta}} \boldsymbol{\psi}^*(\boldsymbol{\theta})$ is the equivalent of solving a linear system of equations, i.e., $\partial_{\boldsymbol{\theta}} \boldsymbol{\psi}^*(\boldsymbol{\theta}) = A^{-1} B$, 
    \begin{equation}
        \begin{split}
            A &= 
            I - \Sigma' (\nabla_{\boldsymbol{\psi}} \Tilde{F}\left(\boldsymbol{\psi}, \boldsymbol{\theta}\right)) \cdot  \nabla_{\boldsymbol{\psi}}^2 \Tilde{F}\left(\boldsymbol{\psi}, \boldsymbol{\theta}\right), \,\text{and}
            \\
            B &= 
            \Sigma'(\nabla_{\boldsymbol{\psi}} \Tilde{F} \left(\boldsymbol{\psi}, \boldsymbol{\theta}\right)) \cdot \partial_{\boldsymbol{\theta}} \nabla_{\boldsymbol{\psi}} \Tilde{F} \left(\boldsymbol{\psi}, \boldsymbol{\theta}\right),
        \end{split}
    \end{equation}
    where 
    $\Sigma' (\boldsymbol{x}) = \diag \left(\begin{bmatrix}\sigma' (x_j)\end{bmatrix}_{j=1}^{|V|}\right)$, 
    and $\sigma'(x) = (1 + \exp{(-x)})^{-2} \cdot \exp{(-x)}$. 
\end{theorem}
The proof is in App.~\fullversion{F
\added{ of \citet{ozcan2024learning}}}{\ref{app:AandB}\deleted{ of the supplement}}. 
Eq.~\eqref{eq:fp_implicit} shows that the Jacobian of the fixed-point solution, $\partial_{\boldsymbol{\theta}} \boldsymbol{\psi}^*(\boldsymbol{\theta})$, can be expressed in terms of Jacobians of $G$ at the solution point. This means implicit differentiation only needs the final fixed point value, whereas automatic differentiation via the approach by 
\citet{ou2022learning} required all the iterates (see also \citep{kolter2020implicit}). In practice, we use JAXopt \citep{blondel2022efficient} for its out-of-the-box implicit differentiation support. \added{This allows us to handle Hessian inverse computations efficiently (see App.~\fullversion{G 
of \citet{ozcan2024learning}}{\ref{app:hess_inv}}}).

\subsection{Implicit Differentiable Mean Field Variation}
Putting everything together, we propose \emph{\textbf{i}mplicitly \textbf{Diff}erentiable \textbf{M}ean \textbf{F}ield variation} ($\mathtt{iDiffMF}$) algorithm. This algorithm finds the solution of the fixed-point in Eq.~\eqref{eq:fixed_point} by a root-finding method. Then, computes the gradient of the loss given in Eq.~\eqref{eq: chain rule} by using the result of the implicit function theorem given in Thm.~\ref{thm:implicit_result}, and updates parameter $\boldsymbol{\theta}$ in the direction of this gradient. We summarize this process in Alg.~\ref{alg:i-diffMF}.

\begin{algorithm}[!t] 
\caption{$\mathtt{iDiffMF}$}\label{alg:i-diffMF}
    \begin{algorithmic}[1]
        \REQUIRE training dataset $\{(S_i^\ast, \, V_i)\}_{i=1}^N$, learning rate $\eta$, number of samples $m$  
        \ENSURE parameter $\boldsymbol{\theta}$
        \STATE $\boldsymbol{\theta} \gets$  initialize
        \REPEAT {
            \STATE sample training data point
                
                $(S^\ast, V) \sim \{(S_i^*,\, V_i)\}_{i=1}^N$
            \STATE initialize the variational parameter
            
                $\boldsymbol{\psi}^{(0)} \leftarrow 0.5 * \mathbf{1}$
    
            \FOR {$j \gets 1, \ldots, |V|$ in parallel}
                \STATE sample $m$ subsets
                
                $S_{\ell} \sim q(S, (\boldsymbol{\psi} | \psi_j \gets 0))$
                \STATE update variational parameter 
                
                $\mathbf{\psi}_j^{*} \gets \sigma \left(\frac{1}{m}\sum_{\ell=1}^m \left[F_{\boldsymbol{\theta}} (S_{\ell} \cup \{j\}) - F_{\boldsymbol{\theta}}(S_{\ell})\right]\right)$
            \ENDFOR
            \STATE update parameter $\boldsymbol{\theta}$ by computing Eq.~\eqref{eq: chain rule} through Thm.~\ref{thm:implicit_result}\label{line:update2}

                $\partial_{\boldsymbol{\theta}} \boldsymbol{\psi}^* (\boldsymbol{\theta}) \gets A^{-1} B$ (see Thm.~\ref{thm:implicit_result})

                $\nabla_{\boldsymbol{\theta}} \mathcal{L} (\boldsymbol{\psi}^{*}, \boldsymbol{\theta}) \gets \nabla_{\boldsymbol{\psi}^*} \mathcal{L}(\boldsymbol{\psi}^*(\boldsymbol{\theta})) \cdot \partial_{\boldsymbol{\theta}} \boldsymbol{\psi}^* (\boldsymbol{\theta})$
    
                $\boldsymbol{\theta} \gets \boldsymbol{\theta} - \eta \nabla_{\boldsymbol{\theta}} \mathcal{L} (\boldsymbol{\psi}^{*}, \boldsymbol{\theta})$
            }
        \UNTIL convergence of $\boldsymbol{\theta}$
    \end{algorithmic}
\end{algorithm}

To emphasize the difference between Alg.~\ref{alg:diffMF} and Alg.~\ref{alg:i-diffMF}, let us focus on lines~\ref{line:update} and~\ref{line:update2}, respectively. On Line~\ref{line:update} of the pseudo-code for the $\mathtt{DiffMF}$ algorithm, gradient of the loss corresponds to 
\begin{align*}
    \nabla_{\boldsymbol{\theta}} \mathcal{L} \left(\boldsymbol{\psi}^{(K)}
    \right) = \nabla_{\boldsymbol{\psi}} \mathcal{L}\left(\boldsymbol{\psi}^{(K)}
    \right) \cdot \partial_{\boldsymbol{\theta}} \boldsymbol{\psi}^{(K)} 
    ,
\end{align*}
where $\boldsymbol{\psi}^{(K)}$ is a nested function in the form of
$$\boldsymbol{\psi}^{(K)} = \boldsymbol{\sigma} (\nabla_{\boldsymbol{\psi}} \Tilde{F} (\ldots (\boldsymbol{\sigma} (\nabla_{\boldsymbol{\psi}} \Tilde{F} (\boldsymbol{\psi}^{(0)}, \boldsymbol{\theta})), \ldots, \boldsymbol{\theta})).$$ Therefore, automatic differentiation has to unroll all $K$ layers during gradient computation. On the other hand, on Line~\ref{line:update2} of the $\mathtt{iDiffMF}$ algorithm, gradient of the loss is computed through Eq.~\eqref{eq: chain rule} where $\partial_{\boldsymbol{\theta}} \boldsymbol{\psi}^*(\boldsymbol{\theta})$ has a closed form formulation as a result of Thm.~\ref{thm:implicit_result}.


\subsection{Complexity}\label{sec:memory}
Reverse mode automatic differentiation has a memory complexity that scales linearly with the number of iterations performed for finding the root of the fixed-point, i.e., it has a memory complexity of $\mathcal{O}(K)$ where $K$ is the total number of iterations \citep{bai2019deep}. On the other hand, reverse mode implicit differentiation has a constant memory complexity, $\mathcal{O}(1)$, because the differentiation is performed \emph{analytically} as a result of using the implicit function theorem. Fig.~\ref{fig:memory_analysis} in Sec.~\ref{sec:exp} reflects the advantage of using implicit differentiation in terms of space requirements numerically.

In the forward mode, the time complexity of the iterative sequence inside $\mathtt{DiffMF}$ is again $\mathcal{O}(K)$ as the number of iterations is pre-selected and does not change with the rate of convergence. Inside $\mathtt{iDiffMF}$, the convergence rate depends on the Lipschitz constant of the fixed-point in Eq.~\eqref{eq:fixed_point} and the size of the ground set. In particular, the number of iterations required for finding the root of Eq.~\eqref{eq:fixed_point} is bounded by $\frac{\log{(\epsilon (1-\omega)/\sqrt{|V|})}}{\log{\omega}},$ where $\epsilon$ is the tolerance threshold and $\omega$ is the Lipschitz constant, i.e., $\|\boldsymbol{\sigma} (\nabla_{\boldsymbol{\psi}} \Tilde{F}\left(\boldsymbol{x}, \boldsymbol{\theta}\right)) - \boldsymbol{\sigma} (\nabla_{\boldsymbol{\psi}} \Tilde{F}\left(\boldsymbol{y}, \boldsymbol{\theta}\right))\|_2 \leq \omega \|\boldsymbol{x} - \boldsymbol{y}\|_2$ (see App.~\fullversion{H
\added{ of \citet{ozcan2024learning}}}{\ref{app:conv_rate}\deleted{ of the supplement}} for computation steps). Thus, the root-finding routine inside $\mathtt{iDiffMF}$ has $\mathcal{O}\left(\frac{\log{(\epsilon (1-\omega)/\sqrt{|V|})}}{\log{\omega}}\right)$ time complexity.

\section{Experiments}\label{sec:exp}

We evaluate our proposed method on five datasets including set anomaly detection, product recommendation, and compound selection tasks (see Tab.~\ref{tab:datasets} for a datasets summary and App.~\fullversion{I
\added{ of \citet{ozcan2024learning}}}{\ref{app:exp_details}\deleted{ of the supplement}} for detailed dataset descriptions). The Gaussian and Moons are synthetic datasets, while the rest are real-world datasets. We closely follow the experimental setup of 
\citet{ou2022learning} 
  w.r.t.~competing algorithm setup, experiments, and metrics.\footnote{\added{\url{https://github.com/neu-spiral/LearnSetsImplicit}}}

\subsection{Algorithms}\label{sec:algs}

We compare three competitor algorithms from~\citep{ou2022learning} to three variants of our $\mathtt{iDiffMF}$ algorithm (Alg.~\ref{alg:i-diffMF}). Additional implementation details are in App.~\fullversion{I
\added{ of \citet{ozcan2024learning}}}{\ref{app:exp_details}\deleted{ of the supplement}}. 
\paragraph{$\mathtt{DiffMF}$~\citep{ou2022learning}:} This is the differentiable mean field variational inference algorithm described in Alg.~\ref{alg:diffMF}. As per Ou et al., we set the number of iterations  to $K = 5$ for all datasets.

\paragraph{$\mathtt{EquiVSet}_{\textmd{ind}}$~\citep{ou2022learning}:} This is the equivariant variational inference algorithm proposed by 
\citet{ou2022learning}. It is a variation of the $\mathtt{DiffMF}$ algorithm where the parameter $\boldsymbol{\psi}$ is predicted by an additional recognition network as a function of the data. As per 
\citet{ou2022learning}, we set  $K = 1$ for all datasets.

\paragraph{$\mathtt{EquiVSet}_{\textmd{copula}}$~\citep{ou2022learning}:} A correlation-aware version of the $\mathtt{EquiVSet}_{\textmd{ind}}$ algorithm where the relations among the input elements are modeled by a Gaussian copula. As per 
\citet{ou2022learning}, we set  $K = 1$ for all datasets.

\paragraph{$\mathtt{iDiffMF}$~(Alg.~\ref{alg:i-diffMF}):} Our proposed implicit alternative to the $\mathtt{DiffMF}$ algorithm where we solve the fixed-point condition in Eq.~\eqref{eq:fixed_point} with a low tolerance threshold ($\epsilon = 10^{-6}$), instead of running the fixed-point iterations in Eq.~\eqref{eq:iterative} for only a fixed number of times. 
Although DNNs are bounded, the exact computation of their Lipschitz constant is, even for two-layer Multi-Layer-Perceptrons (MLP), NP-hard \citep{virmaux2018lipschitz}. 
In our implementation, we use several heuristic approaches to satisfy the condition in Asm.~\ref{asm:bound}. First, we multiply the multilinear relaxation $\Tilde{F}$ by a constant scaling factor ${2}/{(|V| c)}$, treating $c$ as a hyperparameter.  We refer to this as $\mathtt{iDiffMF}_c$. We also consider a dynamic adaptation per batch and fixed-point iteration,  normalizing the gradient of the multilinear relaxation by its 
norm as well as size of the ground set; we describe this heuristic in App.~\fullversion{I.3
\added{ of \citet{ozcan2024learning}}}{\ref{app:normalizing}\deleted{ of the supplement}}. 
We propose two variants, termed $\mathtt{iDiffMF}_2$ and $\mathtt{iDiffMF}_*$, using $\ell_2$ $(\|\cdot\|_2)$ and nuclear $(\|\cdot\|_*)$ norms when scaling, respectively.   


For all algorithms, we use permutation-invariant NN architectures as introduced by Ou et al., described in App.~\fullversion{I.6
\added{ of \citet{ozcan2024learning}}}{\ref{app:archs}}. 
We report all experiment results 
with the best-performing hyperparameters 
based on a $5$-fold cross-validation
. More specifically, we partition each dataset to a training set and a hold out/test set (see Tab.~\ref{tab:datasets} for split ratios). We then divide the training dataset in $5$ folds. 
We identify the best hyperparameter combination through cross-validation across all folds. To produce standard-deviations, we then report the mean and the standard variation of the performance of the 5 models trained under the best hyperparameter combination on the test dataset.

We explore the following hyper-parameters: learning rate $\eta$, number of layers $L$, and different forward and backward solvers. Additional details, including ranges and optimal hyperparameter combinations, can be found in App.~\fullversion{I.7
\added{ of \citet{ozcan2024learning}}}{\ref{app:constant_exp}\deleted{ of the supplement}}.

We use the PyTorch code repository provided by \citet{ou2022learning} for all three competitor algorithms.\footnote{\url{https://github.com/SubsetSelection/EquiVSet}} 
We use the JAX+Flax framework~\citep{jax2018github, frostig2018compiling, flax2020github} 
 for its functional programming abilities for our $\mathtt{iDiffMF}$ implementations. In particular, we implement implicit differentiation using the JAXopt library~\citep{blondel2022efficient}. It offers a modular differentiation tool that can be combined with the existing solvers and it is readily integrated in JAX. We include our code in the supplementary material and will make it public after the review process.

 \begin{table}[!t]
\centering
\resizebox{.9\linewidth}{!}{%
\begin{tabular}{|c|l|c|c|c|c|c|c|c|c|}
\cline{2-10}
\multicolumn{1}{c|}{} & \textbf{Dataset} & $|\Omega|$ & $|\mathcal{D}|$ & $|V|$ & $ \text{avg}(|S^\ast|) $ & $\min (|S^*|)$ & $\max (|S^*|)$ & $d_f$ & Split ratio\\
\cline{2-10}
\multicolumn{1}{c|}{} & CelebA & $202,599$ & $10000$ & $8$ & $2.5$ & $2$ & $3$ & $128$& $11:1$\\
\multicolumn{1}{c|}{} & Gaussian & $100$ & $1000$ & $100$ & $10$ & $10$ & $10$ & $2$& $2:1$\\
\multicolumn{1}{c|}{} & Moons & $100$ & $1000$ & $100$ & $100$ & $10$ & $10$ & $2$& $2:1$\\
\hline
\multirow{12}{*}{\rotatebox[origin=c]{90}{Amazon}} & apparel & $100$ & $4,675$ & $30$ & $4.52$ & $3$ & $19$ & $768$& $2:1$\\
& bath & $100$ & $3,195$ & $30$ & $3.80$ & $3$ & $11$ & $768$ & $2:1$\\
& bedding & $100$ & $4,524$ & $30$ & $3.87$ & $3$ & $12$ & $768$& $2:1$\\
& carseats & $34$ & $483$ & $30$ & $3.26$ & $3$ & $6$ & $768$& $2:1$\\
& diaper    & $100$ & $6,108$ & $30$ & $4.14$ & $3$ & $15$ & $768$& $2:1$\\
& feeding    & $100$ & $8,202$ & $30$ & $4.62$ & $3$ & $2$  & $768$& $2:1$\\
& furniture  & $32$ & $280$ & $30$ & $3.18$ & $3$ & $6$ & $768$& $2:1$\\
& gear      & $100$ &$ 4,277$ & $30$ & $3.8$ & $3$ & $10$ & $768$& $2:1$\\
& health      & $62$ & $2,995$ & $30$ & $3.69$ & $3$ & $9$ & $768$& $2:1$\\
& media      & $58$ & $1,485$ & $30$ & $4.52$ & $3$ & $19$ & $768$& $2:1$\\
& safety      & $36$ & $267$ & $30$ & $3.16$ & $3$ & $5$ & $768$& $2:1$\\
& toys      & $62$ & $2,421$ & $30$ & $4.09$ & $3$ & $14$ & $768$& $2:1$\\
\hline
\multicolumn{1}{c|}{} & BindingDB & $52,273$ & $1,200$ & $300$ & $15$ & $15$ & $15$ & $512$& $11:1$\\ 
\cline{2-10}
\end{tabular}
}
\caption{Summary of the datasets. $\Omega$ denotes the universe with all possible options. $\mathcal{D}$ is the dataset with $(S^\ast, V)$ pairs. $V \subseteq \Omega$ is the ground set of options and $S^\ast$ is the optimal subset of $V$. $d_f$ is the size of the feature vector for each item in $\Omega$. The optimal subset, ground set $(S^\ast, V)$ pair selection/generation is instance specific and we describe these processes in the App.~\fullversion{I
\added{ of \citet{ozcan2024learning}}}{\ref{app:exp_details}\deleted{ of the supplement}}. Optimal subsets are subject to cardinality constraints. 
}
\label{tab:datasets}
\end{table}
 
\subsection{Metrics}

\begin{table*}[!t]
\centering
\resizebox{\linewidth}{!}{%
\begin{tabular}{|l|c|c|c|c|c|c|c|c|c|c|c|}
\cline{2-12}
\multicolumn{1}{c|}{} & \multirow{2}{*}{\textbf{Datasets}} & \multicolumn{2}{c|}{$\mathtt{EquiVSet}_{\textmd{ind}}$} & \multicolumn{2}{c|}{$\mathtt{EquiVSet}_{\textmd{copula}}$} & \multicolumn{2}{c|}{$\mathtt{DiffMF}$} & \multicolumn{2}{c|}{$\mathtt{iDiffMF}_2$}& \multicolumn{2}{c|}{$\mathtt{iDiffMF}_*$} \\
\cline{3-12}
 \multicolumn{1}{c|}{} & & \textbf{Test JC} & \textbf{Time} (s) & \textbf{Test JC} & \textbf{Time} (s) & \textbf{Test JC} & \textbf{Time} (s) & \textbf{Test JC} & \textbf{Time} (s) & \textbf{Test JC} & \textbf{Time} (s) \\
\hline
\multirow{3}{*}{\rotatebox[origin=c]{90}{\makecell{\small AD}}} & CelebA & $55.02\pm0.20$ & $1151.17\pm698.13$ & $56.16\pm0.81$ & $1195.47\pm731.84$ & $54.42\pm0.70$ & $ 1299.13\pm984.20$ & $\underline{56.30\pm0.58}$ & $1880.86\pm266.84$ & $\mathbf{56.55\pm0.49}$ & $1827.76\pm472.78$\\
 & Gaussian & $90.55\pm0.06$ & $30.68\pm3.86$ & $90.94\pm0.09$ & $39.11\pm6.09$ & $90.96\pm0.05$ & $85.75\pm35.82$ & $\underline{90.95\pm0.18}$ & $39.41\pm3.64$ & $\mathbf{91.03\pm0.09}$ & $46.05\pm3.44$\\
 & Moons & $57.76\pm0.11$ & $66.99\pm 4.43$ & $\underline{58.67\pm0.18}$ & $ 62.03\pm6.82$ & $58.45\pm0.15$& $58.24\pm3.01$ & $\underline{58.48\pm0.15}$ & $70.26\pm13.96$ & $\mathbf{58.97\pm0.04}$ & $53.80\pm10.79$\\
\hline
\multirow{12}{*}{\rotatebox[origin=c]{90}{\small PR (Amazon)}} & apparel & $68.45\pm0.96$ & $38.32\pm6.63$ & $\mathbf{78.19\pm0.89}$ & $77.14\pm14.37$ & $70.60\pm1.35$ & $63.06\pm16.12$ & $\underline{76.13\pm4.65}$ & $110.74\pm42.12$ & $73.80\pm5.71$ & $98.43\pm35.07$ \\

& bath & $67.51\pm1.19$ & $34.01\pm5.89$ & $\mathbf{77.72\pm1.98}$ & $53.29\pm6.68$ & $71.87\pm0.27$ & $61.84\pm12.73$ & $\underline{77.68\pm0.98}$ & $70.90\pm12.26$ & $76.43\pm0.81$ & $80.12\pm14.43$ \\
& bedding & $66.20\pm1.10$ & $40.99\pm3.59$ & $\underline{77.26\pm1.24}$ & $67.13\pm12.78$ & $67.66\pm0.39$ & $72.69\pm7.73$ & $\mathbf{77.88\pm0.80}$ & $103.64\pm17.76$& $76.94\pm1.05$ & $88.56\pm19.34$ \\
& carseats & $19.99\pm1.01$ & $12.38\pm4.19$ & $20.03\pm0.15$ & $12.19\pm2.71$ & $20.15\pm0.65$ & $10.53\pm5.01$ & $\underline{21.94\pm1.43}$ & $40.98\pm7.77$ & $\mathbf{22.42\pm1.04}$ & $45.00\pm8.46$ \\
& diaper & $74.26\pm0.73$ & $60.96\pm17.79$ & $\mathbf{83.66\pm0.69}$ & $193.55 \pm 80.28$ & $81.74\pm1.18$ & $95.22\pm10.54$ & $\underline{82.76\pm0.62}$ & $127.09\pm23.58$ & $82.07\pm0.90$ & $144.33\pm30.65$\\
& feeding & $71.46\pm0.43$ & $68.43\pm26.08$ & $\mathbf{82.47\pm0.19}$ & $95.18\pm21.75$ & $77.44\pm0.46$ & $93.27\pm18.81$ & $\underline{81.93\pm1.00}$ & $145.46\pm50.13$ & $81.52\pm1.84$ & $179.65\pm18.05$ \\
& furniture & $17.28\pm0.88$ &  $10.98\pm2.44$ & $17.95\pm0.80$ & $10.03\pm3.23$ & $16.84\pm0.05$ & $9.31\pm1.79$ & $\mathbf{19.93\pm2.68}$ & $34.31\pm6.03$ & $\underline{18.69\pm0.93}$ & $34.30\pm5.96$ \\

& gear & $65.35\pm0.91$ & $40.89\pm3.19$ & $\mathbf{77.33\pm0.90}$ & $69.44 \pm 10.22$ & $66.06\pm2.86$ & $60.95\pm10.38$ & $\underline{73.90\pm10.29}$ & $92.30\pm45.44$ & $73.57\pm6.74$ & $132.10\pm30.14$ \\

& health & $63.04\pm0.41$  & $33.51\pm5.22$ & $72.03\pm0.77$ & $60.18\pm6.31$ & $59.64\pm0.81$ & $51.66\pm2.54$ &  $\mathbf{72.55\pm1.10}$ & $78.65\pm13.01$ & $\underline{72.32\pm1.03}$ & $88.71\pm21.76$ \\
& media & $\mathbf{56.60\pm0.56}$ & $37.45\pm11.06$ & $55.73\pm1.18$ & $45.02\pm4.95$ & $51.32\pm1.11$ & $40.69\pm4.65$ &$\underline{56.39\pm2.68}$ & $65.15\pm18.41$ & $55.58\pm1.75$ & $67.83 \pm 21.18$ \\
& safety & $21.99\pm1.85$ &  $10.39\pm1.87$ & $22.09\pm3.30$ &  $13.14\pm3.13$ & $24.66\pm5.56$ &  $8.59\pm1.31$  & $\mathbf{26.02\pm1.68}$ & $47.66\pm8.62$ & $\underline{25.38 \pm 1.88}$ & $44.63\pm6.45$ \\
& toys & $62.36\pm1.31$ & $34.06\pm6.69$ & $\mathbf{69.08\pm1.04}$ & $47.81\pm9.46$ & $64.39\pm1.64$ & $43.96\pm6.89$  & $68.53\pm1.35$ & $68.34\pm17.82$ & \underline{$68.91 \pm 1.00$} & $80.30 \pm 18.76$ \\
\hline
\rotatebox[origin=c]{90}{\makecell{\small CS}} & BindingDB & $73.59\pm0.75$ & $9934.30 \pm 2591.36$ & $73.57 \pm 2.05$ & $13983.93 \pm 4458.52$ & $73.22\pm 1.08$ & $21472.44 \pm 3239.73$ & $\underline{76.83 \pm 0.50}$ & $10887.64 \pm 1709.79$ & $\mathbf{77.48	\pm 1.04}$ & $10612.98 \pm 946.64$\\
\hline
\end{tabular}
}
\caption{
Test Jaccard Coefficient (JC) and training time for set anomaly detection (AD), product recommendation (PR), and compound selection (CS) tasks, across all five algorithms. $\mathtt{iDiffMF}_2$ and $\mathtt{iDiffMF}_*$ correspond to our algorithm with Frobenius and nuclear norm scaling. \textbf{Bold} and \underline{underline} indicate the best and second-best performance results, respectively. The confidence intervals on the table come from the standard variation of the measurements between folds during cross-validation.}
\label{tab:idiffmf-performance-others}
\end{table*}

\begin{figure*}[!t]
    \centering
    \includegraphics[width=0.94\linewidth]{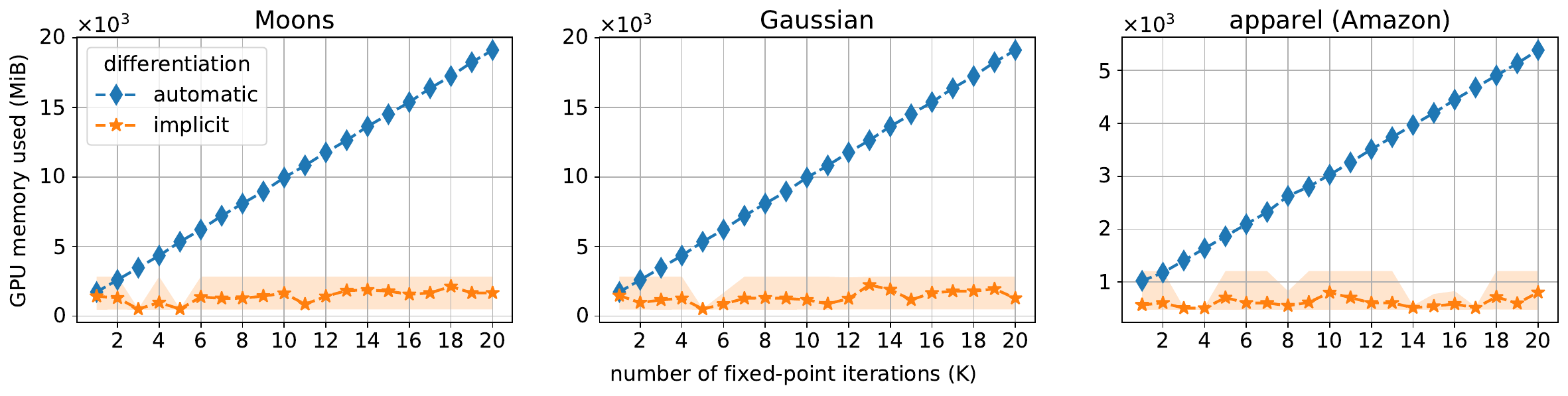}
    \caption{Effects of the choice of differentiation method on the relationship between the allocated GPU memory 
    and the number of fixed-point iterations across different datasets. Blue lines represent automatic differentiation ($\mathtt{DiffMF}$), while the orange lines represent implicit differentiation ($\mathtt{iDiffMF}$). The markers denote the average memory usage. The area between the recorded minimum and maximum memory usage is shaded.}
    \label{fig:memory_analysis}
\end{figure*}

Following 
\citet{ou2022learning}, we measure the performance of different algorithms by
(a) using the trained neural network to predict the optimal subsets corresponding to each query on the test set, and (b) measure the mean Jaccard Coefficient (JC) score across all predictions. 
We describe how the trained objective $F_{\boldsymbol{\theta}}(\cdot)$ is used to produce an optimal subset $\hat{S}^\ast_i$ given query $V_i$ in the test set in App.~\fullversion{I.5
\added{ of \citet{ozcan2024learning}}}{\ref{app:inference}}.

We also measure the running time and the GPU memory usage of the algorithms. During training, we track the amount of memory used every $5$ seconds with the $\texttt{nvidia-smi}$ command while varying the number of maximum iterations. For each number of maximum iterations, we report the minimum, maximum, and average memory usage. 



\subsection{Results}
We report the predictive performance of our proposed $\mathtt{iDiffMF}_2$ and $\mathtt{iDiffMF}_*$ methods against the existing $\mathtt{DiffMF}$ method and its variants on Tab.~\ref{tab:idiffmf-performance-others}, 
and $\mathtt{iDiffMF}_c$ in App.~\fullversion{I.7
\added{ of \citet{ozcan2024learning}}}{\ref{app:constant_exp}\deleted{ of the supplement}}.  
%
%
For the vast majority of the test cases, $\mathtt{iDiffMF}$ variants achieve either the best or the second-best JC score. While the next best competitor, $\mathtt{EquiVSet}_{\textmd{copula}}$, performs the best on some datasets, 
its performance is not  consistent on the remaining datasets, not being even the second best. 
For the Amazon carseats, furniture and safety datasets, $\mathtt{iDiffMF}$ variants give significantly better results than $\mathtt{EquiVSet}_{\textmd{copula}}$, even though $\mathtt{EquiVSet}_{\textmd{copula}}$ is faster. This is probably because $\mathtt{EquiVSet}_{\textmd{copula}}$ converges to a local optimum and finishes training earlier. It is also important to highlight that we evaluate $\mathtt{iDiffMF}$ using JAX+Flax while we use PyTorch to evaluate the baselines. Therefore, the differences in running time can also be explained with the framework differences. Even though $\mathtt{iDiffMF}$ executes fixed-point iterations until convergence, as opposed to $K=1$ or $K=5$ in remaining methods \cite{ou2022learning}, the average running times are comparable across datasets.

In Fig.~\ref{fig:memory_analysis}, we demonstrate the advantages of using implicit differentiation in terms of space complexity. As discussed in Sec.~\ref{sec:memory}, memory requirements remain constant in an interval as the number of fixed-point iterations increases during implicit differentiation. On the contrast, memory requirements increase linearly with the number of iterations during automatic differentiation.

\section{Conclusion}\label{sec:conclusions}
We improve upon an existing learning set functions with an optimal subset oracle setting by characterizing the convergence condition of the fixed point iterations resulting during MLE approximation and by using implicit differentiation over automatic differentiation. Our results perform better than or comparable to the baselines for the majority of the cases without the need of an additional recognition network while requiring less memory. 

\added{
\subsubsection*{Acknowledgments}
We gratefully acknowledge support from the National Science Foundation (grant 1750539).
}


\bibliography{references}

\fullversion{}{
\clearpage
\newpage
\onecolumn
\appendix
\section*{Appendix}
{

\section{Permutation Invariance}\label{app:perm_inv}
In this section, we formally define permutation invariant functions and state the relationship between sum-decomposable and permutation-invariant functions following the works of \citet{zaheer2017deep} and \citet{wagstaff2019limitations}. We use these definitions to explain how we also enforce the property of permutation invariance in this work.
\begin{definition} \cite[Property 1]{zaheer2017deep}
     A function $f(\mathbf{x})$ is \emph{permutation-invariant} if $f(x_1, \ldots, x_M) = f(x_{\pi (1)}, \ldots, x_{\pi (M)})$ for all $\pi$ permutations.
\end{definition}

\begin{definition} \cite[Definition 2.2]{wagstaff2019limitations}
    A set function $f$ is \emph{sum-decomposable} if there are functions $\rho$ and $\phi$ such that
    $$f(S) = \rho \left(\sum_{s \in S} \phi (s)\right).$$
\end{definition}

\begin{theorem} \cite[Theorem 2]{zaheer2017deep} \cite[Theorem 2.8]{wagstaff2019limitations}
    Let $f: 2^S \rightarrow \mathbb{R}$ where $S$ is countable. Then, $f$ is permutation-invariant if and only if it is sum-decomposable via $\mathbb{R}$.
\end{theorem}

In \citet{ou2022learning}, $\phi$ is a dataset specific initial layer that takes set elements, $s$, as inputs and transforms them into some representation $\phi (s)$. These representations are added up and go through $\rho$, a fully connected feed forward neural network. We use the same architectures for $\phi$ and $\rho$ (see App.~\ref{app:archs}). As a result, our model satisfies the permutation invariance property.

\added{
\section{Proof of Equation~\eqref{eq:ELBO}}\label{app:KL-ELBO}
\begin{proof}
    Starting from the definition of the KL divergence, we get:
    \begin{align*}
        \mathbb{KL} (q(S, \boldsymbol{\psi}) || p_{\boldsymbol{\theta}} (S)) &= \sum_{S \subseteq V} q(S, \boldsymbol{\psi}) \log\frac{q(S, \boldsymbol{\psi})}{p_{\boldsymbol{\theta}} (S)} \\
        &= \sum_{S \subseteq V} q(S, \boldsymbol{\psi}) \left(\log q(S, \boldsymbol{\psi}) - \log p_{\boldsymbol{\theta}}(S) \right) \\
        &= \sum_{S \subseteq V}  \left(q(S, \boldsymbol{\psi})\log q(S, \boldsymbol{\psi}) - q(S, \boldsymbol{\psi})\log p_{\boldsymbol{\theta}}(S)\right) \\
        &= \sum_{S \subseteq V}  q(S, \boldsymbol{\psi})\log q(S, \boldsymbol{\psi}) - \sum_{S \subseteq V} q(S, \boldsymbol{\psi})\log p_{\boldsymbol{\theta}}(S) \\ 
        &= -\mathbb{H}(q(S, \boldsymbol{\psi})) - \mathbb{E}_{q(S, \boldsymbol{\psi})}[\log p_{\boldsymbol{\theta}}(S)].\\ 
    \end{align*}
    Observe that, by Eq.~\eqref{eq:EBM}:
    \begin{align*}
    \mathbb{E}_{q(S, \boldsymbol{\psi})}[\log p_{\boldsymbol{\theta}}(S)] = \mathbb{E}_{q(S, \boldsymbol{\psi})}[F_{\boldsymbol{\theta}}(S)] - Z  =  \Tilde{F} (\boldsymbol{\psi}, \boldsymbol{\theta})-Z,
    \end{align*}
    where $Z\equiv\sum_{S' \subseteq V} \exp{(F_{\boldsymbol{\theta}}(S',\, V))}$ does not depend on $\boldsymbol{\psi}$ and thus can be dropped, and $\Tilde{F} (\boldsymbol{\psi}, \boldsymbol{\theta})$ is the multilinear relaxation. 
    Therefore, minimizing the KL divergence w.r.t. $\boldsymbol{\psi}$ is  equivalent to maximizing 
    $\Tilde{F} (\boldsymbol{\psi}, \boldsymbol{\theta}) + \mathbb{H}(q(S, \boldsymbol{\psi})).$ In summary,
    \begin{equation}
    \min_{\boldsymbol{\psi}} \mathbb{KL} (q(S, \boldsymbol{\psi}) || p_{\boldsymbol{\theta}} (S)) \Longleftrightarrow \max_{\boldsymbol{\psi}} \underbrace{\Tilde{F} (\boldsymbol{\psi}, \boldsymbol{\theta}) + \mathbb{H} (q(S, \boldsymbol{\psi}))}_\text{ELBO}.
\end{equation}
\end{proof}
}

\section{Derivations for Sec.~\ref{sec: setup}}
\subsection{Derivation of the Fixed-Point}\label{app:fixed_point_derivation}
Rewriting the ELBO by plugging in the definition of entropy, 
\begin{equation}\label{eq:elbo_w_entropy}
        \Tilde{F}(\boldsymbol{\psi}, \boldsymbol{\theta}) + \mathbb{H} (q(S, \boldsymbol{\psi})) =  \Tilde{F} (\boldsymbol{\psi}, \boldsymbol{\theta}) - \sum_{i=1}^{|V|}\left[\psi_i\log \psi_i + (1-\psi_i)\log (1 - \psi_i)\right].
\end{equation}
Taking the partial derivative of this expression with respect to the $i^{\text{th}}$ coordinate and setting it to zero, yields
\begin{gather*}
    \frac{\partial \Tilde{F} (\boldsymbol{\psi}, \boldsymbol{\theta})}{\partial \psi_i} - \log \frac{\psi_i}{1 - \psi_i} = 0,\\
        \exp{\frac{\partial \Tilde{F} (\boldsymbol{\psi}, \boldsymbol{\theta})}{\partial \psi_i}} = \frac{\psi_i}{1 - \psi_i},\\
        \exp{\frac{\partial \Tilde{F} (\boldsymbol{\psi}, \boldsymbol{\theta})}{\partial \psi_i}} - \psi_i \exp{\left(\frac{\partial \Tilde{F} (\boldsymbol{\psi}, \boldsymbol{\theta})}{\partial \psi_i}\right)} = \psi_i,\\
        \exp{\left(\frac{\partial \Tilde{F} (\boldsymbol{\psi}, \boldsymbol{\theta})}{\partial \psi_i}\right)} = \psi_i \left(1 + \exp{\left(\frac{\partial \Tilde{F} (\boldsymbol{\psi}, \boldsymbol{\theta})}{\partial \psi_i}\right)}\right),\\
        \psi_i = \frac{\exp{\left(\frac{\partial \Tilde{F} (\boldsymbol{\psi}, \boldsymbol{\theta})}{\partial \psi_i}\right)}}{1 + \exp{\left(\frac{\partial \Tilde{F} (\boldsymbol{\psi}, \boldsymbol{\theta})}{\partial \psi_i}\right)}}= \frac{1}{1 + \exp{\left(-\frac{\partial \Tilde{F} (\boldsymbol{\psi}, \boldsymbol{\theta})}{\partial \psi_i}\right)}} = \sigma \left(\frac{\partial \Tilde{F} (\boldsymbol{\psi}, \boldsymbol{\theta})}{\partial \psi_i}\right),
\end{gather*}    
where $\sigma(x) = (1 + \exp{(-x)})^{-1}$ is the sigmoid function.

\subsection{Gradient Computation via Sampling}\label{app:sampling}
\begin{lemma} \label{lem:multigradset}
    Given a set function $F: 2^V \rightarrow \mathbb{R}$ and a vector of probabilities $\boldsymbol{\psi} \in [0, 1]^
    {|V|}$ where $\psi_i = \mathbb{P}[i \in S]$, the gradient of the multilinear relaxation of the set function $F (S)$ is $$\frac{\partial \Tilde{F}(\boldsymbol{\psi})}{\partial{\psi_i}} = \mathbb{E}_{S \sim \boldsymbol{\psi} | \psi_i \leftarrow 0} [F (S + i) - F (S)].$$
\end{lemma}
\begin{proof}
    \begin{equation*}
        \begin{split}
            \frac{\partial \Tilde{F}(\boldsymbol{\psi})}{\partial{\psi_i}} &= \added{\frac{\partial}{\partial \psi_i}} \sum_{S \subseteq V} F(S) \prod_{i \in S} \psi_i \prod_{i \notin S} (1-\psi_i),\\
            &= \mathbb{E}_{S \sim \boldsymbol{\psi} | \psi_i \leftarrow 1} [F(S)] - \mathbb{E}_{S \sim \boldsymbol{\psi} | \psi_i \leftarrow 0} [F(S)],\\
            &= \sum_{S \subseteq V,\, i \in S} F(S) \prod_{j \in S \setminus \{i\}} \psi_j \prod_{j \notin S} (1-\psi_j) - \sum_{S \subseteq V \setminus \{i\}} F(S) \prod_{j \in S} \psi_j \prod_{j \notin S,\, j \neq i} (1-\psi_j),\\
            &= \sum_{S \subseteq V \setminus \{i\}} [F(S + i) - F(S)] \prod_{j \in S} \psi_j \prod_{j \notin S} (1-\psi_j),\\
            &= \mathbb{E}_{S \sim \boldsymbol{\psi} | \psi_i \leftarrow 0} [F(S + i) - F(S)].
        \end{split}
    \end{equation*}
\end{proof}

\added{In the proof above, the fourth line holds because on the first summation of the right-hand side, $i$ is included in all instances of $S$. This is equivalent to iterating over all $S$ that exclude $i$ and then adding $i$ to these sets. In the expectation $\mathbb{E}_{S \sim \boldsymbol{\psi} | \psi_i \leftarrow 0} [F(S + i) - F(S)]$, we are sampling a set $S$ based on $\boldsymbol{\psi}$ where $\psi_i = 0$. Then, we are adding the element $i$ to $S$.}

\begin{corollary}\label{cor:grad}
    Knowing Lemma~\ref{lem:multigradset}, the gradient of the multilinear relaxation $\Tilde{F}(\boldsymbol{\psi}, \boldsymbol{\theta})$, is defined as follows
    \begin{equation}\label{eq:grad}
        \begin{split}
             \frac{\partial \Tilde{F}(\boldsymbol{\psi}, \boldsymbol{\theta})}{\partial{\psi_i}} = \Tilde{F}([\boldsymbol{\psi}]_{+i}, \boldsymbol{\theta}) - \Tilde{F} ([\boldsymbol{\psi}]_{-i}, \boldsymbol{\theta})\replaced{,}{.}
        \end{split}
    \end{equation}
    \added{where $[\boldsymbol{\psi}]_{+i}$, $[\boldsymbol{\psi}]_{-i}$ are operands setting $\psi_i = 1$ and $\psi_i = 0$, respectively.}
\end{corollary}

\added{Note that this derivation is a classic \citep{calinescu2011maximizing} and} Eq.~\eqref{eq:grad} 
can be computed by producing random samples of $S$.

\section{Technical Preliminaries}\label{app:prelim}

\begin{theorem}{(Multivariate Mean Value Theorem \citep{apostol1974mathematical, rudin1976principles, burke2014nonlinear})}\label{thm:MMVT}
    If $\boldsymbol{f}: \mathbb{R}^n \rightarrow \mathbb{R}^m$ is continuously differentiable, then for every $\boldsymbol{x}, \boldsymbol{y} \in \mathbb{R}^n$, there exists a $\boldsymbol{z} \in [\boldsymbol{x}, \boldsymbol{y}]$, such that
    \begin{equation*}
        \|\boldsymbol{f}(\boldsymbol{x}) - \boldsymbol{f}(\boldsymbol{y})\|_2 \leq \sup_{\boldsymbol{z} \in [\boldsymbol{x}, \boldsymbol{y}]} \|\partial \boldsymbol{f}(\boldsymbol{z})\|_{F} \|\boldsymbol{x} - \boldsymbol{y}\|_2,
    \end{equation*}
    where $\|\cdot\|_2$ is the $L_2$ norm and $\|\cdot\|_F$ is the Frobenius norm.
\end{theorem}

\begin{definition}\label{def:contraction}
A mapping $T: X \rightarrow X$ is called a \emph{contraction} on X if there exists a constant $\epsilon \in [0, 1)$ such that for all $x, y \in X$,
    \begin{equation*}
        \|T(x) - T(y)\|_2 \leq \epsilon \|x - y\|_2.
    \end{equation*}   
\end{definition}

\begin{theorem}{(Banach's Fixed Point Theorem~\citep{banach1922operations})}\label{thm:banach_og}
    Let $T: X \rightarrow X$ be a contraction on $X$. Then $T$ has a unique fixed point $x^* \in X$ where $T(x^*) = x^*$.
\end{theorem}

\begin{theorem} \label{thm:implicit}
    (Implicit Function Theorem \citep{krantz2002implicit, blondel2022efficient}) Given a continuously differentiable function $G: \mathbb{R}^n \times \mathbb{R}^d \rightarrow \mathbb{R}^n$, an implicitly defined function $\boldsymbol{x}^*: \mathbb{R}^d \rightarrow \mathbb{R}^n$ of $\boldsymbol{\theta} \in \mathbb{R}^d$, and an optimal solution $\boldsymbol{x}^*(\boldsymbol{\theta})$; let
    \begin{equation}
        G(\boldsymbol{x}^*(\boldsymbol{\theta}), \boldsymbol{\theta}) = 0.
    \end{equation}
    For $(\boldsymbol{x}_0, \boldsymbol{\theta}_0)$ satisfying $G(\boldsymbol{x}_0, \boldsymbol{\theta}_0) = 0$, if the Jacobian $\partial_{\boldsymbol{x}} G$ evaluated at $(\boldsymbol{x}_0, \boldsymbol{\theta}_0)$ is a square invertible matrix, then there exists a function $\boldsymbol{x}^*(\cdot)$ defined on a neighborhood of $\boldsymbol{\theta}_0$ such that 
    $\boldsymbol{x}^*(\boldsymbol{\theta}_0) = x_0$. Furthermore, for all $\boldsymbol{\theta}$ in this neighborhood, we have that $G(x^*(\boldsymbol{\theta}), \boldsymbol{\theta}) = 0$ and $\partial x^*(\boldsymbol{\theta})$ exists. According to the chain rule, the Jacobian $\partial x^*(\boldsymbol{\theta})$ satisfies
    \begin{equation}
        \partial_{\boldsymbol{x}} G(\boldsymbol{x}^*(\boldsymbol{\theta}), \boldsymbol{\theta}) \partial \boldsymbol{x}^*(\boldsymbol{\theta}) + \partial_{\boldsymbol{\theta}} G(\boldsymbol{x}^*(\boldsymbol{\theta}), \boldsymbol{\theta}) = 0.
    \end{equation}
    Therefore, computing $\partial \boldsymbol{x}^*(\boldsymbol{\theta})$ becomes the equivalent of solving the following linear system of equations
    \begin{equation}
        \underbrace{-\partial_{\boldsymbol{x}} G(\boldsymbol{x}^*(\boldsymbol{\theta}), \boldsymbol{\theta})}_{A \in \mathbb{R}^{n \times n}} \underbrace{\partial \boldsymbol{x}^*(\boldsymbol{\theta})}_{J \in \mathbb{R}^{n \times d}} = \underbrace{\partial_{\boldsymbol{\theta}} G(\boldsymbol{x}^*(\boldsymbol{\theta}), \boldsymbol{\theta})}_{B \in \mathbb{R}^{n \times d}}.
    \end{equation}  
\end{theorem}

\section{Proof of Theorem~\ref{thm:banach}}
\label{app:convergence}
Before stating our proof we need to state the following corollary:
\begin{corollary}\label{cor:hessian}
    Knowing Corollary~\ref{cor:grad}, we can write the \emph{Hessian} of the multilinear relaxation as
    \begin{equation}\label{eq:hessian}
        \begin{split}
            \frac{\partial^2 \Tilde{F}(\boldsymbol{\psi}, \theta)} {\partial{\psi}_i \partial{\psi}_j} &= \left(\Tilde{F}([\boldsymbol{\psi}]_{+i, +j}, \boldsymbol{\theta})] - \Tilde{F}([\boldsymbol{\psi}]_{-i, +j}, \boldsymbol{\theta})]\right) - \left(\Tilde{F}([\boldsymbol{\psi}]_{+i, -j}, \boldsymbol{\theta})] - \Tilde{F}([\boldsymbol{\psi}]_{-i, -j}, \boldsymbol{\theta})]\right),\\
            &= \Tilde{F}([\boldsymbol{\psi}]_{+i, +j}, \boldsymbol{\theta})] - \Tilde{F}([\boldsymbol{\psi}]_{-i, +j}, \boldsymbol{\theta}) - \Tilde{F} ([\boldsymbol{\psi}]_{+i, -j}, \boldsymbol{\theta}) + \Tilde{F}([\boldsymbol{\psi}]_{-i, -j}, \boldsymbol{\theta}),
        \end{split}
    \end{equation}
    if $i \neq j$, otherwise $\frac{\partial^2 \Tilde{F}(\boldsymbol{\psi}, \theta)} {\partial{\boldsymbol{\psi}}_i \partial{\boldsymbol{\psi}}_j} = 0$.
\end{corollary}
We proof the following lemma using this corollary:
\begin{lemma}\label{lem:hess_bound}
    Elements of the Hessian given in Eq.~\eqref{eq:hessian} are bounded with $4 \sup_{\boldsymbol{\psi} \in [\boldsymbol{0}, \boldsymbol{1}]} \left|\Tilde{F}\left(\boldsymbol{\psi}, \boldsymbol{\theta}\right) \right|$, i.e., $\sup_{\psi_i, \psi_j \in [0, 1]} \left|\frac{\partial^2 \Tilde{F}\left(\boldsymbol{\psi}, \boldsymbol{\theta}\right)}{\partial_{\psi_i}\partial_{\psi_j}}\right| \leq 4 \sup_{\boldsymbol{\psi} \in [\boldsymbol{0}, \boldsymbol{1}]} \left|\Tilde{F}\left(\boldsymbol{\psi}, \boldsymbol{\theta}\right) \right|$.
\end{lemma}
\begin{proof}
By Corollary~\ref{cor:hessian}, we have
\begin{equation*}
    \begin{split}
        \left|\frac{\partial^2 \Tilde{F}(\boldsymbol{\psi}, \theta)} {\partial{\psi}_i \partial{\psi}_j}\right| &= \left|\left(\Tilde{F}([\boldsymbol{\psi}]_{+i, +j}, \boldsymbol{\theta})] - \Tilde{F}([\boldsymbol{\psi}]_{-i, +j}, \boldsymbol{\theta})]\right) - \left(\Tilde{F}([\boldsymbol{\psi}]_{+i, -j}, \boldsymbol{\theta})] - \Tilde{F}([\boldsymbol{\psi}]_{-i, -j}, \boldsymbol{\theta})]\right)\right|.\\
    \end{split}
\end{equation*}
Using the triangular inequality twice, we get
   \begin{equation*}
    \begin{split}
        \left|\frac{\partial^2 \Tilde{F}(\boldsymbol{\psi}, \theta)} {\partial{\psi}_i \partial{\psi}_j}\right| &\leq \left|\Tilde{F}([\boldsymbol{\psi}]_{+i, +j}, \boldsymbol{\theta})] - \Tilde{F}([\boldsymbol{\psi}]_{-i, +j}, \boldsymbol{\theta})]\right| + \left|\Tilde{F}([\boldsymbol{\psi}]_{+i, -j}, \boldsymbol{\theta})] - \Tilde{F}([\boldsymbol{\psi}]_{-i, -j}, \boldsymbol{\theta})]\right|,\\
        &\leq \left|\Tilde{F}([\boldsymbol{\psi}]_{+i, +j}, \boldsymbol{\theta})]\right| + \left|\Tilde{F}([\boldsymbol{\psi}]_{-i, +j}, \boldsymbol{\theta})]\right| + \left|\Tilde{F}([\boldsymbol{\psi}]_{+i, -j}, \boldsymbol{\theta})]\right| + \left|\Tilde{F}([\boldsymbol{\psi}]_{-i, -j}, \boldsymbol{\theta})]\right|,\\
        &\leq 4 \sup_{\boldsymbol{\psi} \in [\boldsymbol{0}, \boldsymbol{1}]} \left|\Tilde{F}\left(\boldsymbol{\psi}, \boldsymbol{\theta}\right) \right|.
    \end{split}
    \end{equation*} 
\end{proof}
Equipped with this lemma, we are ready to proof Thm.~\ref{thm:banach}:
\begin{proof}
    
    For simplicity, define a mapping $T_{\boldsymbol{\theta}}: [0, 1]^{|V|} \rightarrow [0, 1]^{|V|}$ where $T_{\boldsymbol{\theta}}(\boldsymbol{\psi}) = \boldsymbol{\sigma} (\nabla_{\boldsymbol{\psi}} \Tilde{F}\left(\boldsymbol{\psi}, \boldsymbol{\theta}\right))$
    . 
    Given $\Tilde{F}$ is a polynomial w.r.t. $\boldsymbol{\psi}$ and the sigmoid is a smooth function, $T_{\boldsymbol{\theta}}$ is continuously differentiable w.r.t. $\boldsymbol{\psi}$ in $[0, 1]^{|V|}$. By the multivariate equivalent of the mean-value theorem (see Thm.~\ref{thm:MMVT} in App.~\ref{app:prelim}), for every $\boldsymbol{x}, \boldsymbol{y} \in |0, 1|^{|V|}$, there exists a $\boldsymbol{\psi} \in |0, 1|^{|V|}$, such that
    \begin{equation}\label{eq:mmvt_result}
        \begin{split}
            \|T_{\boldsymbol{\theta}}(\boldsymbol{x}) - T_{\boldsymbol{\theta}}(\boldsymbol{y})\|_2 &\leq \sup_{\boldsymbol{\psi} \in [\boldsymbol{0}, \boldsymbol{1}]} \|\partial T_{\boldsymbol{\theta}}(\boldsymbol{\psi})\|_F \|\boldsymbol{x} - \boldsymbol{y}\|_2,\\
            \|\boldsymbol{\sigma} (\nabla_{\boldsymbol{\psi}} \Tilde{F}\left(\boldsymbol{x}, \boldsymbol{\theta}\right)) - \boldsymbol{\sigma} (\nabla_{\boldsymbol{\psi}} \Tilde{F}\left(\boldsymbol{y}, \boldsymbol{\theta}\right))\|_2 &\leq \sup_{\boldsymbol{\psi} \in [\boldsymbol{0}, \boldsymbol{1}]} \|\partial_{\boldsymbol{\psi}} \boldsymbol{\sigma} (\nabla_{\boldsymbol{\psi}} \Tilde{F}\left(\boldsymbol{\psi}, \boldsymbol{\theta}\right))\|_F \|\boldsymbol{x} - \boldsymbol{y}\|_2.
        \end{split}        
    \end{equation}
    From Eq.~\eqref{eq:A}, we know that $\partial_{\boldsymbol{\psi}} \boldsymbol{\sigma} (\nabla_{\boldsymbol{\psi}} \Tilde{F}\left(\boldsymbol{\psi}, \boldsymbol{\theta}\right)) = \begin{bmatrix} \sigma' \left(\frac{\partial \Tilde{F}\left(\boldsymbol{\psi}, \boldsymbol{\theta}\right)}{\partial_{\psi_i}}\right)\frac{\partial^2 \Tilde{F}\left(\boldsymbol{\psi}, \boldsymbol{\theta}\right)}{\partial_{\psi_i}\partial_{\psi_j}} \end{bmatrix}_{1 \leq i, j \leq |V|}$ where $\sigma'(x) = (1 + \exp{(-x)})^{-2} \cdot \exp{(-x)}$. Then,
    
    \begin{equation*}
    \begin{split}
        \sup_{\boldsymbol{\psi} \in [\boldsymbol{0}, \boldsymbol{1}]} \|\partial_{\boldsymbol{\psi}} \boldsymbol{\sigma} (\nabla_{\boldsymbol{\psi}} \Tilde{F}\left(\boldsymbol{\psi}, \boldsymbol{\theta}\right))\|_F &= \sup_{\boldsymbol{\psi} \in [\boldsymbol{0}, \boldsymbol{1}]} \sqrt{\sum_i^{|V|} \sum_j^{|V|} \left|\sigma' \left(\frac{\partial \Tilde{F}\left(\boldsymbol{\psi}, \boldsymbol{\theta}\right)}{\partial_{\psi_i}}\right)\frac{\partial^2 \Tilde{F}\left(\boldsymbol{\psi}, \boldsymbol{\theta}\right)}{\partial_{\psi_i}\partial_{\psi_j}}\right|^2},\\
        &\leq \sqrt{\sum_i^{|V|} \sum_j^{|V|} \left(\sup_{\psi_i, \psi_j \in [0, 1]} \left|\sigma' \left(\frac{\partial \Tilde{F}\left(\boldsymbol{\psi}, \boldsymbol{\theta}\right)}{\partial_{\psi_i}}\right)\frac{\partial^2 \Tilde{F}\left(\boldsymbol{\psi}, \boldsymbol{\theta}\right)}{\partial_{\psi_i}\partial_{\psi_j}}\right|\right)^2},\\
        &= |V| \sup_{\psi_i, \psi_j \in [0, 1]} \left| \sigma' \left(\frac{\partial \Tilde{F}\left(\boldsymbol{\psi}, \boldsymbol{\theta}\right)}{\partial_{\psi_i}}\right)\frac{\partial^2 \Tilde{F}\left(\boldsymbol{\psi}, \boldsymbol{\theta}\right)}{\partial_{\psi_i}\partial_{\psi_j}} \right|,\\
        &\leq |V| \sup_{\psi_i \in [0, 1]} \left| \sigma' \left(\frac{\partial \Tilde{F}\left(\boldsymbol{\psi}, \boldsymbol{\theta}\right)}{\partial_{\psi_i}}\right)\right| \sup_{\psi_i, \psi_j \in [0, 1]} \left|\frac{\partial^2 \Tilde{F}\left(\boldsymbol{\psi}, \boldsymbol{\theta}\right)}{\partial_{\psi_i}\partial_{\psi_j}}\right|.\\
    \end{split}
\end{equation*}
Since $\argmax_{x \in \mathbb{R}} \sigma'(x) = 0$, we know that $\sup_{\psi_i \in [0, 1]} \left| \sigma' \left(\frac{\partial \Tilde{F}\left(\boldsymbol{\psi}, \boldsymbol{\theta}\right)}{\partial_{\psi_i}}\right)\right| \leq \frac{1}{4}$. Moreover, 
Lemma~\ref{lem:hess_bound} gives us a bound for the elements of the Hessian matrix where $\sup_{\psi_i, \psi_j \in [0, 1]} \left|\frac{\partial^2 \Tilde{F}\left(\boldsymbol{\psi}, \boldsymbol{\theta}\right)}{\partial_{\psi_i}\partial_{\psi_j}}\right| \leq 4 \sup_{\boldsymbol{\psi} \in [\boldsymbol{0}, \boldsymbol{1}]} \left|\Tilde{F}\left(\boldsymbol{\psi}, \boldsymbol{\theta}\right) \right|$. As a result,
\begin{equation}\label{eq:before_asm}
    \sup_{\boldsymbol{\psi} \in [\boldsymbol{0}, \boldsymbol{1}]} \|\partial_{\boldsymbol{\psi}} \boldsymbol{\sigma} (\nabla_{\boldsymbol{\psi}} \Tilde{F}\left(\boldsymbol{\psi}, \boldsymbol{\theta}\right))\|_F \leq |V| \sup_{\boldsymbol{\psi} \in [\boldsymbol{0}, \boldsymbol{1}]} \left|\Tilde{F}\left(\boldsymbol{\psi}, \boldsymbol{\theta}\right) \right|.
\end{equation}
According to Asm.~\ref{asm:bound}, we know that $\sup_{\boldsymbol{\psi} \in [\boldsymbol{0}, \boldsymbol{1}]} \left|\Tilde{F}\left(\boldsymbol{\psi}, \boldsymbol{\theta}\right) \right| < \frac{1}{|V|}$. Therefore, 
\begin{equation}
    \sup_{\boldsymbol{\psi} \in [\boldsymbol{0}, \boldsymbol{1}]} \|\partial_{\boldsymbol{\psi}} \boldsymbol{\sigma} (\nabla_{\boldsymbol{\psi}} \Tilde{F}\left(\boldsymbol{\psi}, \boldsymbol{\theta}\right))\|_F < 1.
\end{equation}
Plugging this in Eq.~\eqref{eq:mmvt_result} above, we get
\begin{equation}
        \begin{split}
        \|\boldsymbol{\sigma} (\nabla_{\boldsymbol{\psi}} \Tilde{F}\left(\boldsymbol{x}, \boldsymbol{\theta}\right)) - \boldsymbol{\sigma} (\nabla_{\boldsymbol{\psi}} \Tilde{F}\left(\boldsymbol{y}, \boldsymbol{\theta}\right))\|_2 &< \|\boldsymbol{x} - \boldsymbol{y}\|_2
        .    
    \end{split}        
\end{equation}
This means 
Eq.~\eqref{eq:fixed_point} is a contraction on $[0, 1]^{|V|}$ (see Definition~\ref{def:contraction} in App.~\ref{app:prelim}). Thus, according to Banach fixed-point theorem (see Thm.~\ref{thm:banach_og} in App.~\ref{app:prelim}) the iterations given in Eq.~\eqref{eq:iterative} are bound to converge to a unique solution.
\end{proof}

\section{Proof of Theorem~\ref{thm:implicit_result}
}\label{app:AandB}
\begin{proof}
For $n=|V|$,
\begin{equation}\label{eq:A}
    \begin{split}
        A &= I - \partial_{\boldsymbol{\psi}} \boldsymbol{\sigma} (\nabla_{\boldsymbol{\psi}} \Tilde{F}\left(\boldsymbol{\psi}, \boldsymbol{\theta}\right)) = I - \begin{bmatrix} \frac{\partial \boldsymbol{\sigma} (\nabla_{\boldsymbol{\psi}} \Tilde{F}\left(\boldsymbol{\psi}, \boldsymbol{\theta}\right))}{\partial {\psi_1}} & \hdots & \frac{\partial \boldsymbol{\sigma} (\nabla_{\boldsymbol{\psi}} \Tilde{F}\left(\boldsymbol{\psi}, \boldsymbol{\theta}\right))}{\partial {\psi_n}} \end{bmatrix},\\
        &= I - \begin{bmatrix} \frac{\partial \sigma_1 (\nabla_{\boldsymbol{\psi}} \Tilde{F}\left(\boldsymbol{\psi}, \boldsymbol{\theta}\right))}{\partial {\psi_1}} & \hdots & \frac{\partial \sigma_1 (\nabla_{\boldsymbol{\psi}} \Tilde{F}\left(\boldsymbol{\psi}, \boldsymbol{\theta}\right))}{\partial {\psi_n}} \\
        \vdots & \ddots & \vdots \\
        \frac{\partial \sigma_n (\nabla_{\boldsymbol{\psi}} \Tilde{F}\left(\boldsymbol{\psi}, \boldsymbol{\theta}\right))}{\partial {\psi_1}} & \hdots & \frac{\partial \sigma_n (\nabla_{\boldsymbol{\psi}} \Tilde{F}\left(\boldsymbol{\psi}, \boldsymbol{\theta}\right))}{\partial {\psi_n}} \end{bmatrix} = I - \begin{bmatrix} \frac{\partial \sigma \left(\frac{\partial \Tilde{F}\left(\boldsymbol{\psi}, \boldsymbol{\theta}\right)}{\partial {\psi_1}}\right)}{\partial {\psi_1}} & \hdots & \frac{\partial \sigma \left(\frac{\partial \Tilde{F}\left(\boldsymbol{\psi}, \boldsymbol{\theta}\right)}{\partial {\psi_1}}\right)}{\partial {\psi_n}} \\
        \vdots & \ddots & \vdots \\
        \frac{\partial \sigma \left(\frac{\partial \Tilde{F}\left(\boldsymbol{\psi}, \boldsymbol{\theta}\right)}{\partial {\psi_n}}\right)}{\partial {\psi_1}} & \hdots & \frac{\partial \sigma \left(\frac{\partial \Tilde{F}\left(\boldsymbol{\psi}, \boldsymbol{\theta}\right)}{\partial {\psi_n}}\right)}{\partial {\psi_n}} \end{bmatrix},\\
        &= I - \begin{bmatrix} \sigma' \left(\frac{\partial \Tilde{F}\left(\boldsymbol{\psi}, \boldsymbol{\theta}\right)}{\partial {\psi_1}}\right) \frac{\partial \Tilde{F}\left(\boldsymbol{\psi}, \boldsymbol{\theta}\right)}{\partial^2 {\psi_1}} & \hdots & \sigma' \left(\frac{\partial \Tilde{F}\left(\boldsymbol{\psi}, \boldsymbol{\theta}\right)}{\partial {\psi_1}}\right)\frac{\partial \Tilde{F}\left(\boldsymbol{\psi}, \boldsymbol{\theta}\right)}{\partial {\psi_1} \partial{\psi_n}} \\
        \vdots & \ddots & \vdots \\
        \sigma' \left(\frac{\partial \Tilde{F}\left(\boldsymbol{\psi}, \boldsymbol{\theta}\right)}{\partial {\psi_n}}\right) \frac{\partial \Tilde{F}\left(\boldsymbol{\psi}, \boldsymbol{\theta}\right)}{\partial {\psi_n} \partial{\psi_1}} & \hdots & \sigma' \left(\frac{\partial \Tilde{F}\left(\boldsymbol{\psi}, \boldsymbol{\theta}\right)}{\partial {\psi_n}}\right) \frac{\partial \Tilde{F}\left(\boldsymbol{\psi}, \boldsymbol{\theta}\right)}{\partial^2 {\psi_n}} \end{bmatrix},\\
        &= I - \begin{bmatrix} \sigma' \left(\frac{\partial \Tilde{F}\left(\boldsymbol{\psi}, \boldsymbol{\theta}\right)}{\partial {\psi_1}}\right) & \hdots & 0 \\
        \vdots & \ddots & \vdots \\
        0  & \hdots & \sigma' \left(\frac{\partial \Tilde{F}\left(\boldsymbol{\psi}, \boldsymbol{\theta}\right)}{\partial {\psi_n}}\right) \end{bmatrix} \begin{bmatrix} \frac{\partial \Tilde{F}\left(\boldsymbol{\psi}, \boldsymbol{\theta}\right)}{\partial^2 {\psi_1}} & \hdots & \frac{\partial \Tilde{F}\left(\boldsymbol{\psi}, \boldsymbol{\theta}\right)}{\partial {\psi_1} \partial{\psi_n}} \\
        \vdots & \ddots & \vdots \\
        \frac{\partial \Tilde{F}\left(\boldsymbol{\psi}, \boldsymbol{\theta}\right)}{\partial {\psi_n} \partial {\psi_1}} & \hdots & \frac{\partial \Tilde{F}\left(\boldsymbol{\psi}, \boldsymbol{\theta}\right)}{\partial^2 {\psi_n}} \end{bmatrix},\\
        &= I - \Sigma' (\nabla_{\boldsymbol{\psi}} \Tilde{F}\left(\boldsymbol{\psi}, \boldsymbol{\theta}\right)) \cdot  \nabla_{\boldsymbol{\psi}}^2 \Tilde{F}\left(\boldsymbol{\psi}, \boldsymbol{\theta}\right),
    \end{split}
\end{equation}
where the function $\Sigma': \mathbb{R}^{n} \rightarrow \mathbb{R}^{n  \times n}$ is defined as $$\Sigma' (\mathbf{x}) = \begin{bmatrix} \sigma' (x_1) & \hdots & 0 \\ \vdots & \ddots & \vdots \\ 0 & \hdots & \sigma' (x_n) \end{bmatrix}$$ and $\sigma': \mathbb{R} \rightarrow \mathbb{R}$ is $$\sigma'(x) = (1 + \exp{(-x)})^{-2} \cdot \exp{(-x)}.$$ 

Similarly,
\begin{equation}
    \begin{split}
        B &= \partial_{\theta} \boldsymbol{\sigma}(\nabla_{\boldsymbol{\psi}} \Tilde{F} (\boldsymbol{\psi}, \boldsymbol{\theta})) = \begin{bmatrix} \frac{\partial \boldsymbol{\sigma} (\nabla_{\boldsymbol{\psi}} \Tilde{F}\left(\boldsymbol{\psi}, \boldsymbol{\theta}\right))}{\partial {\theta_1}} & \hdots & \frac{\partial \boldsymbol{\sigma} (\nabla_{\boldsymbol{\psi}} \Tilde{F}\left(\boldsymbol{\psi}, \boldsymbol{\theta}\right))}{\partial {\theta_d}} \end{bmatrix},\\
        &= \begin{bmatrix} \frac{\partial \sigma_1 (\nabla_{\boldsymbol{\psi}} \Tilde{F}\left(\boldsymbol{\psi}, \boldsymbol{\theta}\right))}{\partial {\theta_1}} & \hdots & \frac{\partial \sigma_1 (\nabla_{\boldsymbol{\psi}} \Tilde{F}\left(\boldsymbol{\psi}, \boldsymbol{\theta}\right))}{\partial {\theta_d}} \\
        \vdots & \ddots & \vdots \\
        \frac{\partial \sigma_n (\nabla_{\boldsymbol{\psi}} \Tilde{F}\left(\boldsymbol{\psi}, \boldsymbol{\theta}\right))}{\partial {\theta_1}} & \hdots & \frac{\partial \sigma_n (\nabla_{\boldsymbol{\psi}} \Tilde{F}\left(\boldsymbol{\psi}, \boldsymbol{\theta}\right))}{\partial {\theta_d}} \end{bmatrix} 
        = \begin{bmatrix} \frac{\partial \sigma \left(\frac{\partial \Tilde{F}\left(\boldsymbol{\psi}, \boldsymbol{\theta}\right)}{\partial {\psi_1}}\right)}{\partial {\theta_1}} & \hdots & \frac{\partial \sigma \left(\frac{\partial \Tilde{F}\left(\boldsymbol{\psi}, \boldsymbol{\theta}\right)}{\partial {\psi_1}}\right)}{\partial {\theta_n}} \\
        \vdots & \ddots & \vdots \\
        \frac{\partial \sigma \left(\frac{\partial \Tilde{F}\left(\boldsymbol{\psi}, \boldsymbol{\theta}\right)}{\partial {\psi_n}}\right)}{\partial {\theta_1}} & \hdots & \frac{\partial \sigma \left(\frac{\partial \Tilde{F}\left(\boldsymbol{\psi}, \boldsymbol{\theta}\right)}{\partial {\psi_n}}\right)}{\partial {\theta_n}} \end{bmatrix},\\
        &= \begin{bmatrix} \sigma' \left(\frac{\partial \Tilde{F}\left(\boldsymbol{\psi}, \boldsymbol{\theta}\right)}{\partial {\psi_1}}\right) \frac{\partial \Tilde{F}\left(\boldsymbol{\psi}, \boldsymbol{\theta}\right)}{\partial {\psi_1} \partial {\theta_1}} & \hdots & \sigma' \left(\frac{\partial \Tilde{F}\left(\boldsymbol{\psi}, \boldsymbol{\theta}\right)}{\partial {\psi_1}}\right)\frac{\partial \Tilde{F}\left(\boldsymbol{\psi}, \boldsymbol{\theta}\right)}{\partial {\psi_1} \partial {\theta_d}} \\
        \vdots & \ddots & \vdots \\
        \sigma' \left(\frac{\partial \Tilde{F}\left(\boldsymbol{\psi}, \boldsymbol{\theta}\right)}{\partial {\psi_n}}\right) \frac{\partial \Tilde{F}\left(\boldsymbol{\psi}, \boldsymbol{\theta}\right)}{\partial {\psi_n} \partial {\theta_1}} & \hdots & \sigma' \left(\frac{\partial \Tilde{F}\left(\boldsymbol{\psi}, \boldsymbol{\theta}\right)}{\partial {\psi_n}}\right) \frac{\partial \Tilde{F}\left(\boldsymbol{\psi}, \boldsymbol{\theta}\right)}{\partial {\psi_n} \partial \theta_d} \end{bmatrix},\\
        &= 
        \begin{bmatrix} \sigma' \left(\frac{\partial \Tilde{F}\left(\boldsymbol{\psi}, \boldsymbol{\theta}\right)}{\partial {\psi_1}}\right) & \hdots & 0 \\
        \vdots & \ddots & \vdots \\
        0  & \hdots & \sigma' \left(\frac{\partial \Tilde{F}\left(\boldsymbol{\psi}, \boldsymbol{\theta}\right)}{\partial {\psi_n}}\right) \end{bmatrix} \begin{bmatrix} \frac{\partial \Tilde{F}\left(\boldsymbol{\psi}, \boldsymbol{\theta}\right)}{\partial {\psi_1} \partial \theta_1} & \hdots & \frac{\partial \Tilde{F}\left(\boldsymbol{\psi}, \boldsymbol{\theta}\right)}{\partial {\psi_1} \partial{\theta_d}} \\
        \vdots & \ddots & \vdots \\
        \frac{\partial \Tilde{F}\left(\boldsymbol{\psi}, \boldsymbol{\theta}\right)}{\partial {\psi_n} \partial{\theta_d}} & \hdots & \frac{\partial \Tilde{F}\left(\boldsymbol{\psi}, \boldsymbol{\theta}\right)}{\partial{\psi_n} \partial \theta_d} \end{bmatrix},\\
        &= \Sigma'(\nabla_{\boldsymbol{\psi}} \Tilde{F} \left(\boldsymbol{\psi}, \boldsymbol{\theta}\right)) \cdot \partial_{\boldsymbol{\theta}} \nabla_{\boldsymbol{\psi}} \Tilde{F} \left(\boldsymbol{\psi}, \boldsymbol{\theta}\right).
     \end{split}
\end{equation}


\end{proof}

\added{\section{Hessian Inverse Computation During Implicit Differentiation}\label{app:hess_inv}
Note first that the Hessian in question can be computed efficiently through Cor.~\ref{cor:hessian}, using only function oracle calls: this is a well-known property of the multilinear relaxation.}

\added{Subsequently, we never have to invert matrices appearing in Eq.~\eqref{eq:fp_implicit} to compute the Jacobian $J$: we need to solve a linear system. Better yet, as highlighted by \citet{blondel2022efficient}, it is not necessary to form the entire Jacobian matrix: all of the readily implemented algorithms are so-called “matrix-free”. It is sufficient to left-multiply or right-multiply the Jacobian by $\partial_{\boldsymbol{\psi}} G$ and $\partial_{\boldsymbol{\theta}} G$ in our notation (see Eq.~\eqref{eq:fp_implicit}: these are respectively called vector-Jacobian product (VJP) and Jacobian-vector product (JVP), which can be computed efficiently \emph{without computing the full Jacobian}. JVPs in turn show up in forward passes while VJPs show up in backward passes of standard back-prop. In particular, left-multiplication (VJP) of $ v^\top = \nabla_{\boldsymbol{\psi}^*} \mathcal{L}(\boldsymbol{\psi}^*(\boldsymbol{\theta}))$ with $J$, can be computed by first solving $A^\top u = v$ for $u$, where $A$ as in Eq.~\eqref{eq:fp_implicit}. Then, $v^\top J$ can be obtained by $v^\top J = u^\top A J = u^\top B$.} 

\added{Finally, we solve these linear systems via the readily implemented conjugate gradient method \citep{hestenes1952methods} ($\mathtt{normalcg}$) or via GMRES \citep{saad1986gmres} ($\mathtt{gmres}$). We treat the choice of backward solver selection as a hyperparameter and optimize w.r.t. different combinations using cross-validation. They are both indirect solvers and iteratively solve the linear system up to a given precision.}

\section{Time Complexity of Finding the Root of the Fixed-Point in Equation~\eqref{eq:fixed_point}}\label{app:conv_rate}

In App.~\ref{app:convergence}, we base our proof on the fact that $\boldsymbol{\sigma} (\nabla_{\boldsymbol{\psi}} \Tilde{F}\left(\boldsymbol{\psi}, \boldsymbol{\theta}\right))$ is a contraction mapping when $\sup_{\boldsymbol{\psi} \in [\boldsymbol{0}, \boldsymbol{1}]} \|\partial_{\boldsymbol{\psi}} \boldsymbol{\sigma} (\nabla_{\boldsymbol{\psi}} \Tilde{F}\left(\boldsymbol{\psi}, \boldsymbol{\theta}\right))\|_F < 1.$ Substituting $q = \sup_{\boldsymbol{\psi} \in [\boldsymbol{0}, \boldsymbol{1}]} \|\partial_{\boldsymbol{\psi}} \boldsymbol{\sigma} (\nabla_{\boldsymbol{\psi}} \Tilde{F}\left(\boldsymbol{\psi}, \boldsymbol{\theta}\right))\|_F < 1$ back in Eq.~\eqref{eq:mmvt_result}, we obtain the Lipschitz constant for the fixed point as $$\|\boldsymbol{\sigma} (\nabla_{\boldsymbol{\psi}} \Tilde{F}\left(\boldsymbol{x}, \boldsymbol{\theta}\right)) - \boldsymbol{\sigma} (\nabla_{\boldsymbol{\psi}} \Tilde{F}\left(\boldsymbol{y}, \boldsymbol{\theta}\right))\|_2 \leq q \|\boldsymbol{x} - \boldsymbol{y}\|_2.$$

For the sequence defined in Eq.~\eqref{eq:iterative}, it holds that
\begin{equation*}
    \|\boldsymbol{\psi}^\ast - \boldsymbol{\psi}^{(K)}\|_2 \leq \frac{q^K}{1-q} \|\boldsymbol{\psi}^{(1)} - \boldsymbol{\psi}^{(0)}\|_2.
\end{equation*}

We wish to stop the iterations when $\frac{q^K}{1-q} \|\boldsymbol{\psi}^{(1)} - \boldsymbol{\psi}^{(0)}\|_2 \leq \epsilon$ and we know that $\boldsymbol{\psi}$ in $[0, 1]^{|V|}$, this means we need to run the iterative sequence given in Eq.~\eqref{eq:iterative} for $K \leq \frac{\log{(\epsilon (1-q)/\sqrt{|V|})}}{\log{q}}$ iterations.

\section{Additional Experiment Details and Results}\label{app:exp_details}
\subsection{Datasets}

\paragraph{Moons and Gaussian.} There are two classes in the synthetic datasets, whose labels are Bernoulli sampled with $p = 0.5$. Based on this label, the optimal subset and ground set pairs are constructed as follows: 1) sampling $10$ points within the class as $S^*$; and 2) sampling $90$ points from the other class as $V \backslash S^*$. This process is repeated until $|\mathcal{D}_{\text{validation + training}}| = 2000$ and $|\mathcal{D}_{\text{test}}| = 1000$.
Following the experimental procedure of 
\citet{ou2022learning}, we use the \textsc{scikit-learn} package~\citep{scikit-learn} to generate the Moons dataset, consisting of two interleaving moons with some noise with variance $\sigma^2=0.1$. For the Gaussian dataset, we sample data from a mixture of Gaussians $\frac{1}{2}\mathcal{N}(\mathbf{\mu}_0, \mathbf{\Sigma})+\frac{1}{2}\mathcal{N}(\mathbf{\mu}_1, \mathbf{\Sigma})$, where $\mathbf{\mu}_0=\left[ \frac{1}{\sqrt{2}}, \frac{1}{\sqrt{2}}\right]$, $\mu_1=-\mu_0$, and $\mathbf{\Sigma}=\frac{1}{4}\mathbf{I}$.

\paragraph{CelebA.}
The CelebFaces Attributes dataset (CelebA)~\citep{liu2015faceattributes} is a large-scale face dataset used extensively in computer vision research, particularly for tasks such as face detection, face attribute recognition, etc. The dataset contains $202,599$ face images of $10,177$ celebrities with various poses and backgrounds. Besides, each image is annotated with $40$ binary attributes, describing facial features and properties (e.g., having a mustache, wearing a hat or glasses, etc.). Following previous work, we select $2$ attributes at random and construct the set $V$ with a size of $8$ and the oracle set $S^*$ with a size of $2$ or $3$, where neither attribute is present (e.g. not wearing glasses and hats). 

\paragraph{Amazon.}
The Amazon Baby Registry dataset~\citep{gillenwater2014amazon} includes various subsets of baby registry products chosen by customers. These products are then organized into $18$ distinct categories. From these, $12$ categories are selected. Each product in the dataset is described by a textual description, which has been transformed into a $768$-dimensional vector using a pre-trained BERT model~\citep{Devlin2019BERTPO}. For each category, the $(V, S^*)$ pairs are sampled using the following process. First, we exclude subsets chosen by customers that contain only one item or more than $30$ items. Next, we divide the remaining subsets into training, validation, and test sets equally. For each oracle subset $S^* \in S$, we randomly sample $30 - |S^*|$ additional products from the same category to ensure that $V$ contains exactly $30$ products. This method constructs a data point $(V, S^*)$ for each customer, simulating a real-world scenario where $V$ represents $30$ products shown to the customer, and $S^*$ represents the subset of products the customer is interested in.

\paragraph{BindingDB.}
BindingDB is a dataset designed to facilitate target selection in drug discovery. Given a set of compounds, the goal is to select the target to which a drug can bind to treat the disease effectively \cite{BindingDB}. An ideal target should possess certain properties such as high biological activity, diversity, absorption, distribution, metabolism, excretion, and toxicology (ADME/Tox) \cite{bhhatarai2019opportunities, singh2024advances}. Thus, one needs to go through multiple filters sequentially, requiring intermediate signals that can be very expensive or impossible due to the privacy policy \cite{ou2022learning}. However, by approaching the target selection from a set function perspective, we only require the compound sets and optimal target pairs. Here, we follow the steps in \citep{ou2022learning}: a) choose two filters: high bioactivity and diversity. b) construct ground set with size $|V|=300$. c) filter out one-third of the compounds with the highest bioactivity via a distance matrix based on the fingerprint similarity of molecules. d) ensure diversity by generating OS oracle $S^*$ according to the centers of clusters which are presented by applying the affinity propagation algorithm. e) split the datasets so that the training, validation, and test set are 1,000, 100, and 100.

\subsection{Hyperparameters}
Both baseline and the proposed models are trained with an Adam optimizer~\citep{kingma2014adam} with a learning rate $\eta \in \{10^{-4}, 10^{-3}, 10^{-2}, 10^{-1}\}$ 
and a batch size of $4$ for BindingDB and $128$ for the remaining datasets. We explore the number of layers, denoted as $L$, as a hyperparameter $L\in \{2, 3\}$.\footnote{The BindingDB dataset requires larger memory for training and, thus, is only tested for 2 layers 
due to computational limits.} For our proposed algorithms, we experiment with different root finding methods, i.e., forward solvers $\in \{\mathtt{fpi}, \mathtt{anderson}\}$ where $\mathtt{fpi}$ corresponds to fixed-point iterations in Eq.~\eqref{eq:iterative} and $\mathtt{anderson}$ corresponds to Anderson acceleration~\citep{anderson1965iterative}. \added{For these root finding methods, we set the tolerance threshold to $10^{-6}$, a very small value, to ensure convergence of fixed-point iterations. This is feasible as convergence is fast.} We also try different methods to solve the linear system of equations (see Thm.~\ref{thm:implicit_result}) that show up during implicit differentiation, i.e., backward solvers in $\in \{\mathtt{normal cg}, \mathtt{gmres}\}$ where they stand for the conjugate gradient~\citep{hestenes1952methods} and GMRES~\citep{saad1986gmres} methods, respectively.   

\subsection{Normalizing the Gradient}\label{app:normalizing}
Scaling the gradient of the multilinear relaxation as discussed in Sec.~\ref{sec:algs}, modifies the fixed-point equation given in Eq.~\eqref{eq:fixed_point}. In particular, when we scale $\nabla_{\boldsymbol{\psi}} \Tilde{F} (\boldsymbol{\psi}, \boldsymbol{\theta})$; we, in effect, solve the following fixed-point equation:
\begin{equation} \label{eq:scaled_fixed_point}
    \begin{split}
        \boldsymbol{\psi} = \boldsymbol{\sigma}\left(\frac{2\nabla_{\boldsymbol{\psi}} \Tilde{F} (\boldsymbol{\psi}, \boldsymbol{\theta})}{|V|Q}\right).
    \end{split}
\end{equation}
We choose $Q=c$ to be a constant for $\mathtt{iDiffMF}_c$ and treat $c$ as a hyperparameter. We report the results in Tab.~\ref{tab:idiffmf-constant-performance}. For $\mathtt{iDiffMF}_2$ and $\mathtt{iDiffMF}_*$, we choose $Q=\|\nabla_{\boldsymbol{\psi}} \Tilde{F} (\boldsymbol{\psi}, \boldsymbol{\theta})\|_2$ and $Q=\|\nabla_{\boldsymbol{\psi}} \Tilde{F} (\boldsymbol{\psi}, \boldsymbol{\theta})\|_*$, respectively. We note that when computing the gradient through implicit differentiation, we take this change into account.

\subsection{Software and Hardware}
We conduct experiments on a DGX Station A100 equipped with four A100 (80GB). The operating system is Ubuntu 22.04.4 LTS with x86-64 architecture, powered by an AMD EPYC 7742 64-Core Processor(4TB of DDR4 memory and a 256MB L3 cache). All experiments are executed using Python 3.9, with PyTorch 2.2.1 and JAX 0.4.26 as the primary software packages. We provide an environment file in the supplement illustrating the version of all other packages.

\subsection{Inference Details} \label{app:inference}
After training, the learned objective $F_{\boldsymbol{\theta}}(\cdot)$ is used to produce an optimal subset $\hat{S}^\ast_i$ given query $V_i$ in the test set. In particular, during inference in the test time, $\boldsymbol{\psi}^{(0)}$ is set to $0.5 * \mathbf{1}$ and one step of the fixed-point iterations is applied to $\boldsymbol{\psi}^{(0)}$ with the learned $F_{\boldsymbol{\theta}}$ to obtain $\boldsymbol{\psi}^\ast$ except for the $\mathtt{DiffMF}$ algorithm. For the $\mathtt{DiffMF}$ algorithm, $K=5$ steps of fixed-point iterations is applied. This is in accordance with the implementation of \citet{ou2022learning}. We also experiment with letting fixed-point iterations converge during inference. We report these results in Tab.~\ref{tab:idiffmf-performance-converging}. Then, the corresponding prediction, $\hat{S_i^\ast}$, is found by applying topN rounding to $\boldsymbol{\psi}^\ast$, i.e., top $|S_i^\ast|$ elements are chosen to be in the prediction set $\hat{S_i^\ast}$.

The mean JC is obtained by computing the Jaccard  coefficient between these prediction sets and the ground truth sets $S_i^\ast$ in the test set. The Jaccard Coefficient of two sets is defined as the average of the size of their intersection divided by the size of their union. Mathematically, for two sets $A$ and $B$, the Jaccard Coefficient $J(A,B)$ is given by 
   $J(A, B) = \frac{|A \cap B|}{|A \cup B|}.$ 
Then, the mean JC 
overall optimal oracle subset pairs are
    $MJC(S_i^\ast,\hat{S_i^\ast}) = \frac{1}{N} \sum_{i=1}^N J(S_i^\ast, \hat{S_i^\ast}).$


\subsection{Architectures}\label{app:archs}
We use the same permutation invariant architecture for the set function $F_{\boldsymbol{\theta}}$ as in \citet{ou2022learning}. For completeness, we also include the architecture information here. All datasets share the same architecture except for their initial layer which encodes the set objects into vector representations. The summary of the set function architecture is given below.
\begin{table}[!t]
    \centering
    \begin{tabular}{c}
    \toprule
     \textbf{Set Function} ($F_{\boldsymbol{\theta}}$)  \\
     \hline
     InitLayer \\
     SumPooling \\
     FC($256$, $500$, ReLU) \\
     FC($500$, $500$, ReLU) \\
     FC($256$, $1$, $-$) \\
     \bottomrule
    \end{tabular}
    \caption{Summary of the architecture of $F_{\boldsymbol{\theta}}$. InitLayer is dataset specific. FC($d_{\text{ input}}$, $d_{\text{ output}}$, $f$) is a fully connected layer with an input dimension of $d_{\text{ input}}$, output dimension of $d_{\text{ output}}$, and activation function $f$.}
    \label{tab:architecture}
\end{table}

\begin{table}[!t]
    \centering
    \resizebox{\linewidth}{!}{%
    \begin{tabular}{c|c|c|c|cc}
    \toprule
     CelebA & Gaussian & Moons & Amazon & \multicolumn{2}{c}{BindingDB}  \\
     \hline
      & \multirow{8}{*}{FC($2$, $256$, $-$)} & \multirow{8}{*}{FC($2$, $256$, $-$)} & \multirow{8}{*}{FC($768$, $256$, $-$)} & \textbf{Drug} & \textbf{Target}\\
     Conv($32$, $3$, $2$, ReLU) & & & & Conv($32$, $4$, $1$, ReLU) & Conv($32$, $4$, $1$, ReLU)\\
     Conv($64$, $4$, $2$, ReLU) & & & & Conv($64$, $6$, $1$, ReLU) & Conv($64$, $8$, $1$, ReLU)\\
     Conv($128$, $5$, $2$, ReLU) & & & & Conv($96$, $8$, $1$, ReLU) & Conv($96$, $12$, $1$, ReLU)\\
     MaxPooling & & & & MaxPooling & MaxPooling \\
     FC($128$, $256$, $-$) & & & & FC($96$, $256$, ReLU) & FC($96$, $256$, ReLU)\\
     & & & & \multicolumn{2}{c}{Concatenation}\\
     & & & & \multicolumn{2}{c}{FC($512$, $256$, $-$)}\\
     \bottomrule
    \end{tabular}}
    \caption{Summary of the initial layer architectures for $F_{\boldsymbol{\theta}}$.}
    \label{tab:initlayers}
\end{table}


\subsection{Additional Results}\label{app:constant_exp}
We report the hyperparameter settings that perform best on the validation sets for each dataset and method for reproducibility purposes. Please see Tab.~\ref{tab:hparams} and Tab.~\ref{tab:idiffmf-constant-performance} for details. 
\begin{table*}[!t]
\centering
\resizebox{\linewidth}{!}{%
\begin{tabular}{|l|c|c|c|c|c|c|c|c|c|c|c|c|c|c|c|}
\cline{2-16}
\multicolumn{1}{c|}{} & \multirow{2}{*}{\textbf{Datasets}} & \multicolumn{2}{c|}{$\mathtt{EquiVSet}_{\textmd{ind}}$} & \multicolumn{2}{c|}{$\mathtt{EquiVSet}_{\textmd{copula}}$} & \multicolumn{2}{c|}{$\mathtt{DiffMF}$} & \multicolumn{4}{c|}{$\mathtt{iDiffMF}_2$}& \multicolumn{4}{c|}{$\mathtt{iDiffMF}_*$} \\
\cline{3-16}
 \multicolumn{1}{c|}{} & & $\eta$ & $L$ & $\eta$ & $L$ & $\eta$ & $L$ & $\eta$ & $L$ & \makecell{forward\\solver} & \makecell{backward\\solver} & $\eta$ & $L$ & \makecell{forward\\solver} & \makecell{backward\\solver} \\
\hline
\multirow{3}{*}{\rotatebox[origin=c]{90}{\makecell{\small AD}}} & CelebA & $0.001$ & $3$ & $0.001$ & $3$ & 0.001 & 3 & 0.01 & 3 & fpi & normal cg & $0.01$ & $3$ & fpi & normal cg\\
 & Gaussian & $0.0001$ & $3$ & $0.0001$ & $3$ & $0.0001$ & $3$ & $0.001$ & $2$ & fpi & normal cg & $0.0001$ & $2$ & fpi & gmres \\
 & Moons & $0.00001$ & $2$ & $0.00001$ & $2$ & $0.0001$& $3$ & $0.001$ & $3$ & anderson & gmres & $0.00001$ & $2$ & fpi & gmres \\
\hline
\multirow{12}{*}{\rotatebox[origin=c]{90}{\small PR (Amazon)}} & apparel & $0.001$ & $3$ & $0.0001$ & $3$ & $0.0001$ & $3$ & $0.0001$ & $3$ & anderson & normal cg & $0.0001$ & $3$ & fpi & normal cg \\
& bath & $0.0001$ & $3$ & $0.0001$ & $3$ & $0.0001$& $3$ & $0.0001$ & $2$ & anderson & normal cg & 0.00001& 3& fpi &normal cg \\
& bedding & $0.0001$ & $3$ & $0.0001$ & $3$ & $0.0001$ & $2$ & $0.0001$ & $2$ & anderson & normal cg & $0.0001$ &	$2$ & fpi & normal cg \\
& carseats & $0.001$ & $2$ & $0.0001$ & $2$ & $0.001$ & $3$ & $0.0001$ & $3$ & fpi & normal cg & $0.0001$ & $3$ & anderson & normal cg\\
& diaper & $0.001$ & $2$ & $0.0001$ & $2$ & $0.0001$ & $2$ & $0.0001$ & $3$ & fpi & normal cg & $0.0001$ & $3$ & fpi &	normal cg\\
& feeding & $0.0001$ & $3$ & $0.0001$ & $3$ & $0.0001$ & $3$ & $0.001$ & $3$ & anderson & normal cg & $0.001$ & $3$ & anderson & gmres\\
& furniture & $0.001$ &  $2$ & $0.001$ & $2$ & $0.001$ & $3$ & $0.0001$ & $2$ & fpi & gmres & $0.0001$ & $3$ & fpi & normal cg \\
& gear & $0.0001$ & $2$ & $0.0001$ & $2$ & $0.0001$ & $2$ & $0.0001$ & $3$ & fpi & normal cg & $0.0001$ & $3$ & fpi & normal cg \\
& health & $0.0001$  & $2$ & $0.0001$ & $3$ & $0.0001$ & $3$ & $0.0001$ & $3$ & anderson & normal cg &  $0.0001$ & $3$ & fpi & normal cg\\
& media & $0.0001$ & $3$ & $0.0001$ & $3$ & $0.001$ & $3$ & $0.0001$ & $3$ & fpi & gmres & $0.0001$ & $3$ & fpi & normal cg\\
& safety & $0.0001$ &  $3$ & $0.001$ &  $3$ & $0.0001$ &  $3$  & $0.0001$ & $3$ & fpi & gmres & $0.0001$ & $3$ & fpi & normal cg \\
& toys & $0.001$ & $3$ & $0.0001$ & $3$ & $0.0001$ & $3$  & $0.0001$ & $3$ & anderson & normal cg & $0.0001$ & $3$ & anderson & normal cg\\
\hline
\rotatebox[origin=c]{90}{\makecell{\small CS}} & BindingDB & $0.0001$ & $2$ & $0.0001$ & $2$ & $0.0001$ & $2$ & $0.0001$ & $2$ & fpi & normal cg & $0.0001$ & $2$ & fpi & normal cg\\
\hline
\end{tabular}
}
\caption{Best performing hyperparameters on average across all folds during cross-validation. We obtain the results on Tab.~\ref{tab:idiffmf-performance-others} based on the hyperparameter combinations we report on this table. We perform grid search on the Cartesian product of $\eta \in \{0.00001, 0.0001, 0.001, 0.01\}$ and $L \in \{2, 3\}$ for the $\mathtt{EquiVSet}_{\textmd{ind}}$, $\mathtt{EquiVSet}_{\textmd{copula}}$, and $\mathtt{DiffMF}$ algorithms, where we perform grid search on the combinations of $\eta \in \{0.00001, 0.0001, 0.001, 0.01\}$, $L \in \{2, 3\}$, forward solver $\in \{\mathtt{fpi}, \mathtt{anderson}\}$, and backward solver $\in \{\mathtt{normal cg}, \mathtt{gmres}\}$ for the $\mathtt{iDiffMF}_2$ and $\mathtt{iDiffMF}_*$ algorithms on all datasets except for BindingDB. Due to its size, we only explore the same range of learning rates for fixed $L$, forward and backward solver choices.
}
\label{tab:hparams}
\end{table*}

\begin{table}[!t]
\centering
\begin{tabular}{|l|c|c|c|c|c|c|c|c|}
\cline{2-9}
\multicolumn{1}{c|}{} & \multirow{2}{*}{\textbf{Datasets}} & \multicolumn{7}{c|}{$\mathtt{iDiffMF}_c$}\\
\cline{3-9}
 \multicolumn{1}{c|}{} & & \textbf{Test JC} & \textbf{Time} (s) & $\eta$ & $L$ & \makecell{forward\\solver} & \makecell{backward\\solver} & $c$\\
\hline
\multirow{3}{*}{\rotatebox[origin=c]{90}{\makecell{\small AD}}} & CelebA & $53.93 \pm 0.63$ & $2211.87 \pm 140.61$ & $0.001$ & $2$ & anderson & normal cg & $10000$\\
 & Gaussian & $90.94 \pm 0.08$ & $60.61 \pm 10.19$ & $0.01$ & $2$ & anderson & normal cg & $1000$\\
 & Moons & $58.37 \pm 0.32$ & $87.62	\pm 20.65$ & $0.001$ & $2$ & anderson & normal cg & $100$\\
\hline
\multirow{12}{*}{\rotatebox[origin=c]{90}{\small PR (Amazon)}} & apparel & $48.08 \pm 0.95$ & $77.97\pm 4.07$ & $0.001$ & $2$ & fpi & normal cg & $100$\\
& bath & $50.17	\pm 0.97$ & $42.46 \pm 2.11$ & $0.00001$ & $2$ & anderson & gmres & $100$\\
& bedding & $47.93 \pm 1.29$ & $132.59 \pm 14.31$ & $0.001$ & $2$ & anderson & normal cg & $10000$\\
& carseats & $20.60 \pm 1.39$ & $38.00 \pm 11.92$ & $0.01$ & $2$ & fpi & normal cg & $1000$\\
& diaper & $57.13 \pm 5.37$ & $100.37 \pm 31.09$ & $0.0001$ & $2$ & anderson & gmres & $100$\\
& feeding & $49.15 \pm 7.20$ & $124.59 \pm 45.12$ & $0.0001$ & $2$ & anderson & gmres & $100$\\
& furniture & $17.65 \pm 0.70$ & $30.77 \pm 3.35$ & $0.01$ & $2$ & fpi & gmres & $100$ \\
& gear & $42.39 \pm 1.00$ & $52.27 \pm 1.15$ & $0.001$ & $2$ & anderson & normal cg & $100$\\
& health & $39.82 \pm 0.34$ & $57.10 \pm 2.88$ & $0.001$ & $2$ & fpi & normal cg & $1000$\\
& media & $41.32 \pm 1.35$ & $46.02 \pm 4.25$ & $0.01$ & $2$ & anderson & normal cg & $100$\\
& safety & $20.50 \pm 2.00$ & $36.76 \pm 6.84$ & $0.01$ & $2$ & fpi & gmres & $100$\\
& toys & $40.87 \pm 2.96$ & $40.00 \pm 3.23$ & $0.01$ & $2$ & fpi & gmres & $100$\\
\hline
\rotatebox[origin=c]{90}{\makecell{\small CS}} & BindingDB & $74.72 \pm 1.73$ & $11121.13 \pm 2553.24$ & $0.001$ & $2$ & anderson & normal cg & $100$\\
\hline
\end{tabular}
\caption{
Test Jaccard coefficient and training time for set anomaly detection (AD), product recommendation (PR) and compound selection (CS) tasks for the $\mathtt{iDiffMF}_c$ algorithm, i.e., the version of $\mathtt{iDiffMF}$ after scaling the objective with a constant $c$. We construct this algorithm by multiplying $\nabla_{\boldsymbol{\psi}} \Tilde{F} (\boldsymbol{\psi}, \boldsymbol{\theta})$ with ${2}/{(|V| c)}.$ 
}
\label{tab:idiffmf-constant-performance}
\vspace{-9pt}
\end{table}

\begin{table*}[!t]
\centering
\resizebox{\linewidth}{!}{%
\begin{tabular}{|l|c|c|c|c|c|c|c|c|c|c|c|}
\cline{2-12}
\multicolumn{1}{c|}{} & \multirow{2}{*}{\textbf{Datasets}} & \multicolumn{2}{c|}{$\mathtt{EquiVSet}_{\textmd{ind}}$} & \multicolumn{2}{c|}{$\mathtt{EquiVSet}_{\textmd{copula}}$} & \multicolumn{2}{c|}{$\mathtt{DiffMF}$} & \multicolumn{2}{c|}{$\mathtt{iDiffMF}_2$}& \multicolumn{2}{c|}{$\mathtt{iDiffMF}_*$} \\
\cline{3-12}
 \multicolumn{1}{c|}{} & & \textbf{Test JC} & \textbf{Time} (s) & \textbf{Test JC} & \textbf{Time} (s) & \textbf{Test JC} & \textbf{Time} (s) & \textbf{Test JC} & \textbf{Time} (s) & \textbf{Test JC} & \textbf{Time} (s) \\
\hline
\multirow{3}{*}{\rotatebox[origin=c]{90}{\makecell{\small AD}}} & CelebA & $55.02\pm0.20$ & $1151.17\pm698.13$ & $56.16\pm0.81$ & $1195.47\pm731.84$ & $54.42\pm0.70$ & $ 1299.13\pm984.20$ &  $55.18\pm0.77$ & $1955.56\pm260.31$ & $56.33\pm0.73$ & $1670.36\pm259.22$\\
 & Gaussian & $90.55\pm0.06$ & $30.68\pm3.86$ & $90.94\pm0.09$ & $39.11\pm6.09$ & $90.96\pm0.05$ & $85.75\pm35.82$ &  $\mathbf{91.0\pm0.02}$ & $46.0\pm5.23$ & $\underline{90.95\pm0.06}$ & $55.76\pm14.16$\\
 & Moons & $57.76\pm0.11$ & $66.99\pm 4.43$ & $\underline{58.67\pm0.18}$ & $ 62.03\pm6.82$ & $58.45\pm0.15$& $58.24\pm3.01$ &  $58.45\pm0.32$ & $77.56\pm16.27$ & $\mathbf{58.95\pm0.1}$ & $57.33\pm10.38$\\
\hline
\multirow{12}{*}{\rotatebox[origin=c]{90}{\small PR (Amazon)}} & apparel & $68.45\pm0.96$ & $38.32\pm6.63$ & $\mathbf{78.19\pm0.89}$ & $77.14\pm14.37$ & $70.60\pm1.35$ & $63.06\pm16.12$ &  $\underline{76.06\pm4.56}$ & $121.7\pm52.21$ & $73.73\pm5.91$ & $139.77\pm58.75$ \\

& bath & $67.51\pm1.19$ & $34.01\pm5.89$ & $\mathbf{77.72\pm1.98}$ & $53.29\pm6.68$ & $71.87\pm0.27$ & $61.84\pm12.73$ &  $\underline{77.49\pm1.19}$ & $79.65\pm11.73$ & $75.51\pm0.45$ & $151.23\pm22.02$ \\

& bedding & $66.20\pm1.10$ & $40.99\pm3.59$ & $\mathbf{77.26\pm1.24}$ & $67.13\pm12.78$ & $67.66\pm0.39$ & $72.69\pm7.73$ &  $\underline{77.07\pm2.02}$ & $118.84\pm40.73$ & $76.93\pm1.05$ & $93.63\pm18.62$ \\

& carseats & $19.99\pm1.01$ & $12.38\pm4.19$ & $20.03\pm0.15$ & $12.19\pm2.71$ & $20.15\pm0.65$ & $10.53\pm5.01$ & $\underline{22.24\pm1.44}$ & $45.82\pm8.58$&  $\mathbf{22.42\pm1.04}$ & $54.50\pm12.11$  \\

& diaper & $74.26\pm0.73$ & $60.96\pm17.79$ & $\mathbf{83.66\pm0.69}$ & $193.55 \pm 80.28$ & $81.74\pm1.18$ & $95.22\pm10.54$ &  $\underline{82.75\pm0.63}$ & $154.05\pm27.68$ & $81.65\pm0.74$ & $184.36\pm57.70$\\

& feeding & $71.46\pm0.43$ & $68.43\pm26.08$ & $\mathbf{82.47\pm0.19}$ & $95.18\pm21.75$ & $77.44\pm0.46$ & $93.27\pm18.81$ &  $\underline{81.54\pm1.80}$ & $165.30\pm45.35$ & $81.52\pm1.84$ & $245.12\pm55.21$ \\

& furniture & $17.28\pm0.88$ &  $10.98\pm2.44$ & $17.95\pm0.80$ & $10.03\pm3.23$ & $16.84\pm0.05$ & $9.31\pm1.79$ &  $\mathbf{20.29\pm2.50}$ & $42.42\pm6.82$ & $\underline{18.69\pm0.93}$ & $36.90\pm6.05$\\

& gear & $65.35\pm0.91$ & $40.89\pm3.19$ & $\mathbf{77.33\pm0.90}$ & $69.44 \pm 10.22$ & $66.06\pm2.86$ & $60.95\pm10.38$ & $\underline{73.91\pm10.30}$ & $112.19\pm56.94$ & $73.47\pm10.81$ & $127.13\pm56.82$ \\
& health & $63.04\pm0.41$  & $33.51\pm5.22$ & $72.03\pm0.77$ & $60.18\pm6.31$ & $59.64\pm0.81$ & $51.66\pm2.54$ &  $\underline{72.14\pm1.39}$ & $95.60\pm25.53$ & $\mathbf{72.31\pm1.04}$ & $116.89\pm32.01$\\
& media & $\mathbf{56.60\pm0.56}$ & $37.45\pm11.06$ & $55.73\pm1.18$ & $45.02\pm4.95$ & $51.32\pm1.11$ & $40.69\pm4.65$ & $\underline{56.41\pm2.63}$ & $73.1\pm21.26$ & $55.20\pm1.62$ & $67.3\pm6.66$\\
& safety & $21.99\pm1.85$ &  $10.39\pm1.87$ & $22.09\pm3.30$ &  $13.14\pm3.13$ & $24.66\pm5.56$ &  $8.59\pm1.31$  &  $\mathbf{25.98\pm1.73}$ & $51.08\pm7.64$ & $\underline{25.73\pm2.33}$ & $47.96\pm6.39$\\
& toys & $62.36\pm1.31$ & $34.06\pm6.69$ & $\mathbf{69.08\pm1.04}$ & $47.81\pm9.46$ & $64.39\pm1.64$ & $43.96\pm6.89$  & $68.55\pm1.13$ & $79.46\pm23.40$ & \underline{$68.70\pm1.01$} & $109.88\pm34.23$ \\
\hline
\rotatebox[origin=c]{90}{\makecell{\small CS}} & BindingDB & $73.59\pm0.75$ & $9934.30 \pm 2591.36$ & $73.57 \pm 2.05$ & $13983.93 \pm 4458.52$ & $73.22\pm 1.08$ & $21472.44 \pm 3239.73$ & \underline{$76.97\pm0.74$} & $13004.75\pm5611.39$ & $\mathbf{77.42\pm0.64}$ & $12097.14\pm1985.58$\\
\hline
\end{tabular}
}
\caption{
Test Jaccard Coefficient (JC) and training time for set anomaly detection (AD), product recommendation (PR), and compound selection (CS) tasks, across all five algorithms. $\mathtt{iDiffMF}_2$ and $\mathtt{iDiffMF}_*$ correspond to our algorithm with Frobenius and nuclear norm scaling. This table differs from Tab.~\ref{tab:idiffmf-performance-others} in its $\mathtt{iDiffMF}$ columns. In this table, fixed-point iterations are run until convergence during inference. \textbf{Bold} and \underline{underline} indicate the best and second-best performance results, respectively. The confidence intervals on the table come from the standard variation of the measurements between folds during cross-validation.}
\label{tab:idiffmf-performance-converging}
\end{table*}
}
}
\end{document}